\documentclass[12pt]{article} % For LaTeX2e
\usepackage{times}
 \usepackage[left=2cm,top=2cm,right=2cm]{geometry}

\usepackage[utf8]{inputenc} % allow utf-8 input
\usepackage{adjustbox}

\usepackage[T1]{fontenc}    % use 8-bit T1 fonts
\usepackage{url}            % simple URL typesetting
\usepackage{booktabs}       % professional-quality tables
\usepackage{amsfonts}       % blackboard math symbols
\usepackage{nicefrac}       % compact symbols for 1/2, etc.
\usepackage{microtype}      % microtypography
\usepackage{xspace}
\usepackage{amsmath, amsthm}
\usepackage{algorithm}
\usepackage{algorithmic}
\usepackage{color}
\usepackage{enumitem}
\usepackage{comment}
\usepackage{bm}

\usepackage[scaled=.9]{helvet}

\usepackage{url}
\usepackage{authblk}
\usepackage[dvipsnames]{xcolor}
\usepackage[colorlinks=true, linkcolor=blue, citecolor=blue]{hyperref}
\usepackage{cleveref}

\newtheorem{remark}{Remark}

\definecolor{LightCyan}{rgb}{0.88,1,1}
\usepackage{caption}

\usepackage{natbib}
\bibpunct{(}{)}{;}{a}{,}{,}

\newcount\Comments  % 0 suppresses notes to selves in text
\definecolor{darkgreen}{rgb}{0,0.5,0}
\definecolor{darkred}{rgb}{0.7,0,0}
\definecolor{teal}{rgb}{0.3,0.8,0.8}
\definecolor{orange}{rgb}{1.0,0.5,0.0}
\definecolor{purple}{rgb}{0.8,0.0,0.8}
\newcommand{\kibitz}[2]{\ifnum\Comments=1{\textcolor{#1}{\textsf{\footnotesize #2}}}\fi}

\def\tr{\mathop{\text{tr}}\kern.2ex}

\long\def\comment#1{}

\def\tr{\mathop{\text{Tr}}}

\def\cS{{\mathcal{S}}}

\def\cX{{\mathcal{X}}}
\def\cD{{\mathcal{D}}}
\def\cP{{\mathcal{P}}}

\def\cM{{\mathcal{M}}}
\def\cZ{{\mathcal{Z}}}
\def\cA{{\mathcal{A}}}
\def\cF{{\mathcal{F}}}
\def\cO{{\mathcal{O}}}

\def\tr{{\text{Tr}}}

\def\*#1{\bm{#1}} 
\def\+#1{\mathcal{#1}} 
\def\-#1{\mathrm{#1}} 
\def\^#1{\mathbb{#1}} 
\def\$#1{\mathtt{#1}}
\def\@#1{\mathsf{#1}}

\def\dual{\@D(\bm{\alpha},\bm{\beta})}

\newcommand{\bel}{\begin{eqnarray}\label}
\newcommand{\eel}{\end{eqnarray}}
\newcommand{\bes}{\begin{eqnarray*}}
\newcommand{\ees}{\end{eqnarray*}}

\usepackage{mathrsfs}
\usepackage{colortbl}

\usepackage{fullpage}
\usepackage{hyperref}
\usepackage{breakcites}
\def\##1\#{\begin{align}#1\end{align}}
\def\$#1\${\begin{align*}#1\end{align*}}
\usepackage{def1}

\begin{document}
\title{Provably Efficient Representation Learning with Tractable Planning in Low-Rank POMDP}
\author[1]{\small Jiacheng Guo}
\author[1]{\small Zihao Li}
\author[3]{\small Huazheng Wang }
\author[1]{\small Mengdi Wang}
\author[4]{\small Zhuoran Yang}
\author[1]{\small Xuezhou Zhang}
\affil[1]{\small Princeton University, \{jg9901,zl9045,mengdiw,xz7392\}@princeton.edu}

\affil[3]{\small Oregon State University, huazheng.wang @oregonstate.edu}
\affil[4]{\small Yale University, zhuoran.yang@yale.edu}
\date{}

\maketitle

\begin{abstract}
    In this paper, we study representation learning in partially observable Markov Decision Processes (POMDPs), where the agent learns a decoder function that maps a series of high-dimensional raw observations to a compact representation and uses it for more efficient exploration and planning. 
    We focus our attention on the sub-classes of \textit{$\gamma$-observable} and \textit{decodable POMDPs}, for which it has been shown that statistically tractable learning is possible, but there has not been any computationally efficient algorithm. We first present an algorithm for decodable POMDPs that combines maximum likelihood estimation (MLE) and optimism in the face of uncertainty (OFU) to perform representation learning and achieve efficient sample complexity, while only calling supervised learning computational oracles. We then show how to adapt this algorithm to also work in the broader class of $\gamma$-observable POMDPs.
    
\end{abstract}
\section{Introduction}
Markov Decision Processes (MDPs) are commonly used in reinforcement learning to model problems across a range of applications that involve sequential decision-making. However, MDPs assume that the agent has perfect knowledge of the current environmental state, which is often not realistic. To address this, Partially Observable Markov Decision Processes (POMDPs) have been introduced as an extension of MDPs \citep{cassandra1998survey,murphy2000survey, braziunas2003pomdp}. In POMDPs, the agent does not have direct access to the environmental state. Instead, it receives observations that are sampled from a state-dependent distribution. POMDPs are an important class of models for decision-making that is increasingly being used to model complex real-world applications, ranging from robotics and navigation to healthcare, finance, manufacturing \citep{roy2006planning,chen2016pomdp,ghandali2018pomdp,liu2020adaptive}.

POMDPs differ from MDPs in the presence of \textit{non-Markovian} observations. In the MDP setting, we assume that the system follows a Markovian law of transition induced by the agent's policy, meaning that the current state of the environment depends solely on the previous state and the action taken. As a result, there exists a Markovian optimal policy, i.e. a policy whose action only depends on the current state. However, with the assumption of unobserved hidden states and therefore the lack of Markovian property, generally, the optimal policy 
of POMDP depends on the full history. This causes great difficulties in both computational and statistical aspects, as the agent has to maintain a long-term \textit{memory} while performing learning and planning. In fact, the agent can only learn the posterior distribution of the latent state given the whole history, known as the belief state, and then maps the belief state to the action. Therefore, learning POMDP is well-known to be intractable in general, even with a small state and observation space \citep{10.2307/3689975, 10.1145/347476.347480,golowich2022planning}. 

Nevertheless, this doesn't rule out special problem structures that enable statistically efficient algorithms. In particular, \citet{katt2018bayesian} and \citet{liu2022partially} study a setting called $\gamma$-\textit{observable} POMDP, i.e. POMDP whose omission matrix is full-rank, and achieve polynomial sample complexities. Another more tractable class is called the \textit{decodable} POMDP, in which we can decode the latent state by $L$-step back histories \citep{efroni2022provable}. Subsequent works on provably efficient learning in POMDP have since found broader classes of statistically tractable problem instances that generalized the above setting \cite{zhan2022pac,uehara2022provably,liu2022optimistic}. 
One common drawback of the aforementioned works, however, is their computational intractability. Almost all of these algorithms follow the algorithmic template of \textit{optimistic planning}, where the algorithm iteratively refines a version space of plausible functions in the function class. To achieve this, these methods perform an iteration and elimination procedure over all remaining functions in the function class in every iteration of the algorithm, making them unimplementable for large function classes. In this paper, we take a representation learning approach to overcome this issue.

Recent works have found that by performing representation learning and using the learned feature to construct exploration bonuses, one can achieve efficient reinforcement learning in rich-observation environments both in theory \citep{doi:10.1126/science.aar6404,pmlr-v139-stooke21a,uehara2021representation} and in practice \citep{guo2022byol}.
However, learning representation is more subtle in POMDPs. The main challenge is two-folded: (1) To efficiently learn a near-optimal policy, the agent must be able to construct a representation of the history that compactly captures the belief state. This causes difficulty in computation, as the agent will suffer from the curse of \textit{long memory} \citep{golowich2022learning,liu2022partially}. (2) Without direct knowledge of the underlying state, it is hard to design an exploration bonus that balances \textit{exploration and exploitation}.
Motivated by the above challenges, the main inquiry we aim to answer in this study is:
 \begin{center}
        \emph{Can we design a representation-learning-aided RL algorithm for POMDPs that is both statistically efficient and amenable to efficient implementation?}
\end{center} 

In this work, we give an affirmative answer to the question. 
% The significance of our work is that our algorithm is sample efficient.
% It is also computationally tractable in the sense that it only requires a supervised learning oracle in MLE. OFU is efficiently achieved by LSVI-UCB. \mw{unfinished paragraph?}
We consider the subclass of \textit{$\gamma$-observabile} and \textit{decodable} POMDP. To handle (1), we consider a short memory learning scheme that considers a fixed-size history instead of the whole history, which is common in POMDP planning \citep{golowich2022planning}. To handle (2), we follow the principle of \textit{optimism in the face of uncertainty (OFU)}, which chooses actions according to the upper bound of the value function. Such a method is typical for balancing exploration and exploitation trade-offs in online MDP learning \citep{cai2020provably,jin2020provably,yang2020reinforcement,ayoub2020model}. We also remark that our UCBVI-type algorithm is computationally efficient, as it avoids intractable optimistic planning (e.g., UCRL \citep{auer2008near}). Specifically, we present a representation learning-based reinforcement learning algorithm for solving POMDPs with low-rank latent transitions, which is common for representation learning in MDP cases \citep{agarwal2020flambe,uehara2021representation, zhan2022pac}.

\begin{table*}[th!]
\centering
\captionsetup{font=footnotesize}

\begin{adjustbox}{width=\textwidth}
\begin{tabular}{|c|c|c|c|}
\hline
Algorithm & Setting & Sample Complexity & Computation\\\hline

\citet{liu2022partially} & Tabular & $\text{poly}(S, A, O, H, \gamma^{-1})/\epsilon^{2}$ & Version Space\\\hline

\citet{golowich2022learning} & Tabular & $\text{quasi-poly}(S, A, O, H, (OA)^{\gamma^{-1}},\epsilon^{-1})$ & $\text{quasi-poly}(S, A, O, H, (OA)^{\gamma^{-1}},\epsilon^{-1})$\\\hline

\citet{uehara2022provably} & Low-rank & $\text{quasi-poly}(H,d, A^{c/\gamma}, \log|\mathcal{M}|)/\epsilon^{2}$ & Version Space\\\hline

\citet{zhan2022pac} & Low-rank & $\text{poly}(H,d, A,1/\gamma,\log|\mathcal{M}|,\log O)/\epsilon^{2}$ & Version Space\\\hline

\citet{wang2022embed} & Low-rank$^*$ & $\text{quasi-poly}(H,d, A^{c/\gamma},\log|\mathcal{M}|)/\epsilon^{2}$ & Version Space\\\hline

\rowcolor{LightCyan}\textsc{PORL$^2$} (ours) & Low-rank & $\text{quasi-poly}(H,d, A^{c/\gamma},\log|\mathcal{M}|)/\epsilon^{2}$ & MLE Oracle+LSVI-LLR\\\hline
\end{tabular}
\end{adjustbox}
\caption{Comparison to all existing works that solve $\gamma$-observable POMDP with formal sample complexity guarantees. Among them, \citet{liu2022partially} and \citet{golowich2022learning} study the tabular setting and their algorithm cannot easily incorporate function approximations. \citet{golowich2022learning} proposed the only algorithm for POMDP that is known to have a formal computational complexity guarantee, and yet the algorithm relies heavily on iterating over the observation space, which cannot be done when the observation space is large or even infinite. Among the algorithms that handle low-rank observable POMDPs, that of \citet{zhan2022pac} escapes the exponential dependency on $\gamma$ at a cost of an explicit dependency on the size of the observation space $O$. Computationally, all existing algorithms other than \citet{golowich2022learning} fall into the category of version-space learners, where the algorithm must keep track of the set remaining plausible functions in the function class and eliminate the ones that are inconsistent with existing observations. Such a procedure will have a computational complexity scale linearly with the size of the function class, which is generally considered to be inefficient, and particularly not amendable for modern neural-network-based implementation. In contrast, our algorithm only relies on calling an MLE computational oracle, which is standard in supervised learning and amendable to efficient implementation
%\footnote{see e.g. \cite{zhang2022making} for a modern implementation using contrastive learning} 
and calling the LSVI-LLR algorithm, which has a poly$(H,A,d)$ computational complexity.}

\label{c}
\end{table*}

\section{Related Work}
Our work is built upon the bodies of literature on both (i) reinforcement learning in POMDPs and (ii) representation learning. In this section, we will focus on the related works in these two directions.
\paragraph{Learning in POMDPs} Provably efficient RL methods for POMDPs have been studied in a number of recent works \citep{li2009multi,guo2016pac,katt2018bayesian,jin2020sample,jafarnia2021online,liu2022partially}. Even when the underlying dynamics of the POMDP are known, simply planning is still hard and relies on \textit{short-memory approximation} \citep{papadimitriou1987complexity}. Moreover, when learning POMDP, the estimation of the model is computationally hard \citep{mossel2005learning}, and learning the POMDPs is statistically intractable \citep{krishnamurthy2016pac}.

However, this doesn't rule out the possibility of finding an efficient algorithm for a particular class for POMDP. A line of work studies a wide sub-classes of POMDPs and has achieved positive results.  \citet{guo2016pac} and \citet{azizzadenesheli2016reinforcement} use spectral methods to learn POMDPs and obtain polynomial sample complexity results without addressing the strategic exploration. \cite{jafarnia2021online} uses the posterior sampling technique to learn POMDPs in the Bayesian setting, with time, the posterior distribution will converge to the true distribution.

The \textit{observability assumption}, or the \textit{weakly-revealing} assumption, has been widely studied for learning in POMDPs. It assumes that the distribution on observations space can recover the distribution on the latent states. It is a very rich subset as it contains the tractable settings in \citep{guo2016pac, jin2020provably} such as the overcomplete setting.  By incorporating observability and assuming a computation oracle such as optimistic planning, \citep{katt2018bayesian,liu2022partially}  achieve favorable polynomial sample complexities. At the same time, \citet{golowich2022planning} proposed algorithms that can achieve quasi-polynomial sample and computational complexity. 

Another common assumption used in POMDPs is the \textit{decodability} assumption, which assumes we can reveal the latent state by $L$-step back histories \citep{efroni2022provable}. \citet{efroni2022provable} obtained polynomial sample complexities. It can be regarded as a generalization of block MDPs \cite{krishnamurthy2016pac}, in which the latent state can be uniquely determined by the current observation. More generally, in an $L$-step \textit{decodable} POMDP,  the latent state can be uniquely decoded from
the most recent history (of observations and actions) of a short length $L$. We remark on the existing literature in block MDPs or decodable POMDPs (e.g. \citet{krishnamurthy2016pac,du2019provably,misra2020kinematic,efroni2022provable,liu2022partially})

\paragraph{Representation Learning in RL}
There has been considerable progress on provable representation learning
for RL in recent literature \citep{https://doi.org/10.48550/arxiv.2006.01107,agarwal2020flambe,DBLP:journals/corr/abs-2102-07035,https://doi.org/10.48550/arxiv.2110.04652}. \citet{DBLP:journals/corr/abs-2102-07035,DBLP:journals/corr/abs-2202-00063} present the model-free approach base on block MDPs and \citep{uehara2021representation,agarwal2020flambe} study the model-based approach through MLE on general low-rank MDP setting. These works focus on fully observable settings, 
% and assume that the model of the  \citep{cai2020provably,jin2021bellman}. \xz{what do you mean by "state feature"? The reader may confuse it with feature learning in representation learning. In general, avoid ambiguous terminologies unless clearly defined in context.} Algorithms developed for this case
only consider an environment with a Markovian transition. For POMDPs, \citet{wang2022embed} learn the representation with a constant past sufficiency assumption, and their sample complexity has an exponential dependence on that constant. \citet{https://doi.org/10.48550/arxiv.2206.12020} uses an actor-critic style algorithm to capture the value link functions with the assumption that the value function is linear in historical trajectory, which can be too strong in practice. \citet{zhan2022pac,liu2022partially} uses MLE to construct a confidence set for the model of POMDP or PSR, and \citet{wang2022embed} assumes a density estimation oracle that controls the error between the estimated model and the real one.  However, both algorithms use optimistic planning, which is computationally inefficient. In comparison, our algorithm achieves optimism by a UCB-type algorithm, and the only necessary oracle is MLE, which is more amenable to computation. To ensure short memory, \citet{wang2022embed} make a constant past sufficiency assumption, which means that the trajectory density from the past $L$ steps could determine the current state distribution. Therefore the sample complexity has an exponential dependence on that constant. As a comparison, our result for $\gamma$-observable POMDPs is more general and achieves better sample complexity. In Table \ref{c}, we compare our work to all prior works that achieve sample efficient learning in $\gamma$-observation POMDPs.

\section{Preliminaries}
\paragraph{Notations}
For any natural number $n\in \NN$, we use $[n]$ to denote the set $\{1,\cdots,n\}$. For vectors we use $\|\cdot\|_p$ to denote $\ell_p$-norm, and we use $\|x\|_{A}$ to denote $ \sqrt{x^\top Ax}$. For a set $\cS$ we use $\Delta(\cS)$ to denote the set of all probability distributions on $\cS$. For an operator $\mathbb{O}: \cS \rightarrow \RR$ and $b\in \Delta(\cS)$, we use   $\mathbb{O} b : \mathcal{O} \rightarrow \RR$ to denote $\int_{\cS}  \mathbb{O}(o\mid s)b(s)\operatorname{ds} $. For two series $\{a_n\}_{n\geq 1}$ and $\{b_n\}_{n\geq 1}$, we use $a_n \leq O(b_n)$ to denote that there exists $C>0$ such that $a_n \leq C\cdot b_n$. 

\textbf{POMDP Setting} In this paper we consider a finite horizon \textit{partially observable Markov decision process (POMDP)}  $\mathcal{P}$, which can be specified as a tuple $\mathcal{P}=\left(\mathcal{S}, \mathcal{A}, H, \mathcal{O}, d_0, \{r_h\}_{h=1}^H, \{\mathbb{P}_h\}_{h=1}^H, \{\mathbb{O}_h\}_{h=1}^H\right)$. Here $\mathcal{S}$ is the state space, $\mathcal{A}$ is a finite set of actions, $H \in \mathbb{N}$ is the episode length,  $\mathcal{O}$ is the set of observations, $d_0$ is the known initial distribution over states. For a given $h\in[H]$, 
$\mathbb{P}_h:\mathcal{S}\times\mathcal{A}\rightarrow \mathcal{S}$ is the transition  kernel and $r_h: \cO \rightarrow [0,1]$ is the reward function at step $h$. For each $a \in \mathcal{A}$, we abuse the notation and use $\mathbb{P}_h(a) $ to denote the probability transition kernel over the state space conditioned on action $a$ in step $h$. $\mathbb{O}_h:\cS\rightarrow \Delta(\cO)$ is the  observation distribution at step $h$, where for $s \in \mathcal{S}, o \in \mathcal{O}$, $\mathbb{O}_h(o\mid s)$ is the probability of observing $o$ while in state $s$ at step $h$. We denote $r=(r_1, \ldots, r_H)$,  $\mathbb{P}=(\mathbb{P}_1,\cdots,\mathbb{P}_H)$ and  $\mathbb{O}=(\mathbb{O}_1,\cdots,\mathbb{O}_H)$.  

 We remark that \textit{Markov Decision Process} (MDP) is a special case of POMDP, where $\cO=\cS$ and $\mathbb{O}_h(o\mid s)=\ind\{o=s\}$ for all $h\in[H],~o\in\cO$ and $s\in \cS$. 
 %In this work, we assume the observation function also has a low-rank structure: $\mathbb{O}^*_h(o\mid s)=O^{*}(o)^{\top}X(s)$ for all $o,s$. 
\paragraph{Low-rank transition kernel} In our work, we focus on the problem where the underlying transition of the environment has a low-rank structure.
\begin{definition}[Low-rank transition \citep{agarwal2020flambe,uehara2021representation}]
A transition kernel $\mathbb{P}_h:\cS\times\cA\rightarrow \cS$ admits a low-rank decomposition of dimension $d$ if there exists two mappings $\omega_h^{*}: \mathcal{S} \rightarrow \mathbb{R}^d \text {, and } \psi_h^{*}: \mathcal{S} \times \mathcal{A} \rightarrow \mathbb{R}^d$ such that $\mathbb{P}_h(s' \mid s, a)=\omega_h^*(s')^{\top}\psi_h^*(s,a)$. 
\end{definition}

\paragraph{Extended POMDP} For notation convenience in the proof, we define the extended POMDP to allow $h<0$. Specifically, we will extend a POMDP from step $h=3-2L$ for a suitable choice of $L$. This particular choice is for the proof only and does not affect the execution of the algorithm in any way. The agent's interaction with the extended POMDP starts with an initial distribution $d_0$. 
For $s,s'\in \cS$, $a\in \cA$ and $h\leq 0$, we define $\phi^*_{h}(s, a)=d_0$, $\mu(s')=d_0(s')e_1$. Hence for a fixed constant $L>0$, all the dummy action and observation sequences $\{o_{3-2L},a_{3-2L},\cdots,o_0,a_0\}$ leads to the same initial state distribution $d_0$. 
A general policy $\pi$ is a tuple $\pi=(\pi_{3-2L}, \ldots, \pi_H)$, where $\pi_h: \mathcal{A}^{2L+h-3} \times \mathcal{O}^{2L+h-2} \rightarrow \Delta(\mathcal{A})$ is a mapping from histories up to step $h$, namely tuples $(o_{3-2L: h}, a_{3-2L: h-1})$, to actions. For  $L$ we will denote the collection of histories up to step $h$ by $\mathcal{H}_h:=\mathcal{A}^{2L+h-3} \times \mathcal{O}^{2L+h-3}$ and the set of policies by $\Pi^{\text {gen}}$, meaning that $\Pi^{\text {gen}}=\Delta\bigl(\prod_{h=1}^H \mathcal{A}^{\mathcal{H}_h}\bigl).$
For $z_h=(o_{h-L+1:h}, a_{h-L+1:h-1})$, we denote $z_{h+1}=c(z_h,a_h,o_{h+1})=(o_{h-L+2:h+1},a_{h-L+2:h})$. 
For any policy $\pi,~h\in[H]$, and positive integer $n$, we use $\pi_h \circ_{n}U(\cA)$ to denote the policy that $(\pi_h \circ_{n}U(\cA))_i=\pi_i$ for $i\leq h-n$ and $(\pi_h \circ_{n}U(\cA))_i=U(\cA)$ for $i\geq h-n+1$, which takes the first $h-n$ actions from $\pi$ and takes remaining actions uniformly.

\paragraph{$L$-memory policy class}In our work, we consider $L$-memory policies. For all $h\in[H]$, let $\cZ_h=\cO^{L}\times\cA^{L-1}$. An element $z_h\in \cZ_h$ is represented as $z_h=(o_{h+1-L:h},a_{h+1-L:h-1})$. A $L$-memory policy is a tuple $\pi=(\pi_{3-2L}, \ldots, \pi_H)$, where $\pi_h: \cZ_h \rightarrow \Delta(\mathcal{A})$.  We define the set of $L$-memory policy by $\Pi^L$. We remark that $\Pi^L\subset \Pi^{\text{gen}}$.

We define the value function for $\pi$ at step 1 by $V_1^{\pi, \mathcal{P}}(o_1)=\mathbb{E}_{o_{2: H} \sim \pi}^{\mathcal{P}}[\sum_{h=1}^H r_h(o_h)\mid o_1]$, namely as the expected reward received by following $\pi$. 

% We also introduce the following performance metrics $$
% \operatorname{Regret}(K) = \sum_{k=1}^K V_1^{\pi^*} - V_1^{\pi^k},
% $$
% \mw{main results are not regret bound. is the definition needed?}
% here $\pi^*\in \Pi^{\text{gen}}$ \mw{in your theorem optimal pi is from $\Pi^L$. inconsistent} is the optimal policy which maximize the value function $V_1^\pi(o_1)$.

%%%%%%%%
\section{Warmup: Low-rank POMDPs with $L-$step decodability}

To begin with, let us use the $L$-step decodable POMDP setting to motivate our algorithm.

\begin{assum}[Low-Rank $L$-step decodability]\label{ass_decodability_linear} 
For all $h\in[H]$, $s_{h+1}\in \cS$ and $(z_h,a_h)\in \cZ_h\times \cA$. There exists $\phi^*_h$ such that $\mathbb{P}_h(s_{h+1}\mid z_h,a_h)=\phi_h^*(z_h,a_h)^{\top}\omega_h^*(s_{h+1})$, where $\omega^*_h$ is the same function as given in Definition 1. 
\end{assum}

Note that Assumption \ref{ass_decodability_linear} is more general than the decodability assumption made in recent works \citep{DBLP:journals/corr/abs-2202-03983,https://doi.org/10.48550/arxiv.2206.12020}. They assume that a suffix of length $L$ of the history suffices to predict the latent state, i.e., there is some decoder $\mathbf{x}:\cZ_i\rightarrow \cS, i\in[H]$ such that $s_h=\mathbf{x}(z_h)$. It leads to our assumption that $\mathbb{P}(s_{h+1}\mid z_h,a_h)=\psi_h^*(\mathbf{x}(z_h),a_h)^{\top}\omega_h^*(s_{h+1})$, but not vice versa.

With Assumption \ref{ass_decodability_linear}, for any $z_h=(o_{h+1-L: h}, a_{h+1-L: h-1})$, $a_h$ and $o_{h+1}$, we have 
\begin{align}\label{eq:mdp_low_decom_L}
&\PP_h(o_{h+1} \mid o_{h+1-L: h}, a_{h+1-L: h}) \\ 
&\qquad\quad
= \Big[\int_{\cS'}\omega^*_h(s')^{\top}\mathbb{O}_{h+1}(o^*_{h+1}\mid s')\operatorname{d}s' \Big] \phi^*_h(z_h,a_h).\nonumber
\end{align}

For any $o_h\in \cO$, we denote \#\mu^*_h(o_h)=\int_{\cS'}\omega^*_h(s')^{\top}\mathbb{O}^*_{h}(o_{h}\mid s')\operatorname{d}s'\label{eq:construct mu}\#

When considering state spaces of arbitrary size, it is necessary to apply function approximation to generalize across states. Representation learning is a natural way to allow such generalization by granting the agent access to a function class of candidate embeddings. Towards this end, we make a commonly used realizability assumption that the given function class contains the true function. 
\begin{assum}[Realizability]\label{ass_realizability_deco}
For all $h\in[H]$, there exists a known model class $\cF=\{(\mathbb{O}_h,\omega_h,\phi_h):\mathbb{O}_h\in \Gcal,\omega_h \in \Omega, \phi_h\in \Phi\}_{h=1}^H$, where $\mathbb{O}_h^*\in\mathcal{G},~\omega_h^*\in \Omega$ and $\phi_h^*\in \Phi$ for all $h\in[H]$.
\end{assum}
We compare this assumption to the work of \citet{DBLP:journals/corr/abs-2202-03983}, which also studies in $L$-decodable POMDPs. Their learning in the function approximation sets assumes that the agent could get access to a function class that contains the $Q^*$ function. Our assumption is more realizable since it is easier to have access to the transition class than the class of $Q^*$ function.
\subsection{Algorithm Description}
In this section, we present our algorithm 
\underline{P}artially \underline{O}bservable  \underline{R}epresentation \underline{L}earning-based \underline{R}einforcement \underline{L}earning
(PORL$^2$).

In Algorithm \ref{PORL_decode}, the agent operates in the underlying POMDP environment $\cP$. In each episode, the agent takes three steps to update the new policy: (i) using a combination of the latest policy $\pi$ and uniform policy $U(\cA)$ to collect data for each time step $h\in[H]$, (ii) calling the MLE oracle to learn and updating the representation $\widehat{\phi}_h$, and (iii) calculating an exploration bonus $\widehat{b}_h$ and applying Least Square Value Iteration for $L$-step Low-Rank POMDPs (LSVI-LLR, cf. Algorithm \ref{LSVI_low_rank}) to update our policy with the combined reward $r+\widehat{b}.$

\begin{algorithm}[t]
    \begin{algorithmic}[1]
    \caption{Partially Observable Representation Learning for $L$-decodable POMDPs (PORL$^2$-decodable)}
    \label{PORL_decode}
        \REQUIRE Representation classes $\{\cF_h\}_{h=0}^{H-1}$, parameters $K$, $\alpha_k$, $\lambda_k$
        \STATE Initialize policy $\pi^0=\{\pi_0, \ldots, \pi_{H-1}\}$ to be arbitrary policies and replay buffers $\cD_h=\varnothing,~\cD'_h=\varnothing$ for all $h$.
        \FOR{$ k \in[K]$} 
            \STATE\label{data}{Data collection from $\pi^{k-1},~\forall h\in [H],$
            $
            \tau_h\sim d^{\pi_h^{k-1}\circ_{L} U(\cA)}_{\cP};~\cD_h=\cD_h\cup
            \{\tau_h\}$,~
            $\tilde{\tau}_h\sim d^{\pi_h^{k-1}\circ_{2L}U(\cA)}_{\cP};~\cD'_h=\cD'_h\cup\{\tilde{\tau}_h\}.
            $
            %&o_{2-L:h-L-1}:a_{2-L:h-L}\sim d_{\cP}^{\pi^{k-1}},~a_{h-L+1:h}\sim U(\cA),~o_{i+1}\sim \mathbb{P}_h^{\cP}(a_{i-L+1:i}:o_{i-L+1:i})\\&~\text{for}~i\in [h-L-1,h-1];~~~\cD_h=\cD_h\cup\{a_{h-L+1:h}:o_{h-L+1:h+1}\}\\& \tilde{o}_{2-L:h-L-2}:\tilde{a}_{2-L:h-2L}\sim d_{\cP}^{\pi^{k-1}},~\tilde{a}_{h-2L+1:h}\sim U(\cA),~\tilde{o}_{i+1}\sim \mathbb{P}_h^{\cP}(\tilde{a}_{i-L+1:i}:\tilde{o}_{i-L+1:i})\\&~\text{for}~i\in[h-L-2,h-1];~~~\cD'_h=\cD'_h\cup\{\tilde{a}_{h-L+1:h}:\tilde{o}_{h-L+1:h+1}\}\$
            }
            \STATE   Learn representations for all $h\in [H]$:
            %$(\widehat{\phi}_h^k,\widehat{\mu}_h^k)=\argmax_{(\phi_h,\mu_h)\in \cF}$ $\mathbb{E}_{\cD\cup\cD'}[\log\phi_h(z_h,a_h)^{\top}\mu_h(o_{h+1})]$.
           $(\widehat{\mathbb{O}}_h^k,\widehat{\omega}_h^k,\widehat{\phi}_h^k)=\argmax_{(\mathbb{O}_h,\omega_h,\psi_h)\in \cF}$ $\mathbb{E}_{\cD\cup\cD'}[\log\phi_h(z_h,a_h)^{\top}\mu_h(o_{h+1})],$ where $\mu_h$ is computed by \eqref{eq:construct mu}.

            \STATE{ Define exploration bonus for all $h\in [H]$:
            $\widehat{b}_h^k(z,a)=\min\bigg\{\alpha_k\sqrt{\widehat{\phi}_h^k(z,a)^{\top}\Sigma_h^{-1}\widehat{\phi}^k_h(z,a)},2\bigg\},
            $
            with $\Sigma_h:=\sum_{z\sim\cD_h}\widehat{\phi}_h^k(z,a)\widehat{\phi}_h^k(z,a)^{\top}+\lambda_{k}I$.
            }\label{bonus}
            \STATE{Set $\pi^k$ as the policy returned by calling Algorithm 2:
            $
            \operatorname{LSVI-LLR}(\{r_{h}+\widehat{b}_{h}^{k}\}_{h=0}^{H-1},\{\widehat{\phi}_{h}^{k}\}_{h=0}^{H-1},\{\widehat{\mu}_{h}^{k}\}_{h=0}^{H-1},\{\mathcal{D}_{h} \cup \mathcal{D}_{h}'\}_{h=0}^{H-1}, \lambda_{k}).
            $
            }
        \ENDFOR
        \STATE{\textbf{Return} $\pi^0,\cdots,\pi^K$}
    \end{algorithmic}
\end{algorithm}
The data collection process has two rounds. In the first round, the agent rollouts the current policy $\pi^{k-1}$ for the first $h-L$ and takes $U(\cA)$ for the remaining steps
 to collect data. In the second round, the agent rollouts $\pi^{k-1}$ for the first $h-2L$ and takes $U(\cA)$ for the remaining steps. After collecting new data and concatenating it with the existing data, the agent learns the representation by calling the MLE oracle on the historical dataset (Line 3). Then, we set the exploration bonus based on the learned representation (Line 4), and we update the policy with the learned representations and bonus-enhanced reward (Line 5).

To update our policy, we apply \textsc{LSVI-LLR} -- an adaptation of the classic LSVI algorithm to L-step low-rank POMDPs. For a given reward $r$ and a model $(\mu, \phi)$, the probability $\mathbb{P}^{\cP}_h(c(z,a,o')\mid z,a)=\mu_h(o')^{\top}\phi_h(z,a)$. Therefore, we have $Q_h(z,a)=r_h(o)+\sum_{o'\in \cD}\phi_h(z,a)^{\top}\mu_h(o_{h+1}')V_{h+1}(z_{h+1})$, where $z_{h+1}=c(z_h,a_h,o_h')$. After inductively computing the Q-function, we can out the greedy policy $\pi^t_h = \argmax_{a}Q_h(z_h,a)$.

 \begin{remark}[Computation]
     Regarding the computational cost, our algorithm only requires calling MLE computation oracle $H$ times in every iteration. Optimism is achieved by adding a bonus to the reward function, which takes $O(Hd^2)$ flops in each iteration to compute with the Sherman-Morrison formula. Importantly, we avoid the optimistic planning procedure that requires iterating over the whole function class $\cF$. Then the time complexity is dominated by the LSVI-LLR step 5 $\pi_h(z) = \operatorname{argmax}_{a\in\Acal}Q_h(z,a)$ for all $(z,a)\in\Dcal^k$, which causes a  $O(AHd^2K)$ running time in every iteration and a total $O(AHd^2K^2)$ running time. Therefore, our algorithm is much more amendable to a practical implementation.
 \end{remark}

\begin{algorithm}[t]
    \begin{algorithmic}[1]
    \caption{Least Square Value Iteration for $L$-step Low-Rank POMDPs (LSVI-LLR)}
    \label{LSVI_low_rank}
        \REQUIRE $\{r_h\}_{h=0}^{H-1}$ 
        , features $\{\phi_h\}_{h=0}^{H-1},\{\mu_{h}\}_{h=0}^{H-1}$, datasets $\{\mathcal{D}_h\}_{h=0}^{H-1}$,  regularization $\lambda$.
        \STATE Initialize $V_H(z)=0$ for any $z\in\Zcal$.
        \FOR{$h=H-1\rightarrow 1$} 
            %\STATE{$\Sigma_h=\sum_{z, a \in \mathcal{D}_h} \phi_h(z, a) \phi_h(z, a)^{\top}+\lambda I$.}
            \STATE{For $(z_h,a_h)\in \cZ_h\times \cA_h$, set \$Q_h(z_h,a_h)&=r_h(o_h)+\sum_{o_{h+1}\in \cD}\phi_h(z_h,a_h)^{\top}\mu_h(o_{+1})\\&\qquad\quad \cdot V_{h+1}(c(z_h,a_h,o_{h+1})),\$
            where $z_h=(o_{h-L+1:h},a_{h-L+1:h-1})$.}
            \STATE{Set $V_h(z)=\max_{a\in \cA}Q_h(z,a)$.}
            %\STATE{Set $\overline{w}_h=\Sigma_h^{-1} \sum_{s'\in \mathcal{D}_h}\overline{V}_{h+1}(s')$.}
            %\STATE{Set $\underline{w}_h\Sigma_h^{-1} \sum_{s'\in \mathcal{D}_h}\underline{V}_{h+1}(s')$}
            %\STATE{Set $\overline{Q}_h(s,a)=\overline{w}_h^{\top}\phi_h(s,a)+r_h(s,a)$, and $\overline{V}_h(s)=\max_a \overline{Q}_h(s, a)$.}
            %\STATE{Set $\underline{Q}_h(s,a)=\underline{w}_h^{\top}\phi_h(s,a)+r_h(s,a)$, and $\underline{V}_h(s)=\min_a \underline{Q}_h(s, a)$.}
            %\FOR{$s'\in \cD$}
                %\STATE{$\widehat{\mu_{s'}}=\Sigma_h^{-1}\sum_{s, a,s'\in \mathcal{D}_h} \phi_h(s, a)$.}
                %\STATE{Set $A=\{s'\mid\widehat{\mu{s'}}\notin\mathbb{R}^{d}_+\}$.}
                
                %\STATE{Set $\overline{Q}_h(s,a)=\overline{Q}_h(s,a)+\ind\{\phi(s,a)^{\top}\mu(s')<0\}\Sigma_h^{-1}(\underline{V}_{h+1}(s')-\overline{V}_{h+1}(s'))$, and $\overline{V}_h(s)=\max_a \overline{Q}_h(s, a)$.}
                %\STATE{Set $\underline{Q}_h(s,a)=\underline{Q}_h(s,a)+\ind\{\phi(s,a)^{\top}\mu(s')<0\}\Sigma_h^{-1}(\overline{V}_{h+1}(s')-\underline{V}_{h+1}(s'))$, and $\underline{V}_h(s)=\min_a \underline{Q}_h(s, a)$.}
            %\ENDFOR
            
            %\STATE{$w_h=\Sigma_h^{-1} \sum_{s'\in \mathcal{D}_h}V_{h+1}(s')\sum_{s, a,s'\in \mathcal{D}_h} \phi_h(s, a) \ind\{\Sigma_h^{-1}\sum_{s, a,s'\in \mathcal{D}_h} \phi_h(s, a)}\in \mathbb{R}^{d}_+\}$.}
            %\STATE{$\text { Set } Q_h(s, a)=w_h^{\top} \phi_h(s, a)+r_h(s, a) \text {, and } V_h(s)=\max _a Q_h(s, a)$.
          \STATE{$\text {Set } \pi_h(z)=\operatorname{argmax}_{a\in \cA} Q_h(z, a)$.
            }
        \ENDFOR
        \STATE{\textbf{Return} $\pi=\{\pi_0,\cdots,\pi_{H-1}\}$.}
    \end{algorithmic}
\end{algorithm}
%Let $\widehat{\phi}=\operatorname{REPLEARN}(\mathcal{D}, \Phi_h, \mathcal{F}_h, \lambda_k, T_k, \ell_k)$. 

\subsection{Analysis}

We have the following guarantee of our Algorithm \ref{PORL_decode}.
\begin{theorem}[Sample complexity of PORL$^2$-$L$-decodable]\label{theorem:sample}
Under Assumption \ref{ass_decodability_linear} and Assumption \ref{ass_realizability_deco}, for  fixed $\delta, \epsilon \in(0,1)$, and let $\overline{\pi}$ be a uniform mixture of $\pi^0, \ldots, \pi^{K-1}$. By setting the parameters as
$$
\begin{aligned}
\alpha_k&=\tilde{\Theta}(\sqrt{k|\mathcal{A}|^{L} \zeta_{k}+
\lambda_{k} d+k \zeta_{k}}),\\
\lambda_k&=\Theta(d \log (|\cF|k/\delta)), \\
\end{aligned}
$$
with prob. at least $1-\delta$, we have $V_1^{\cP,\pi^{*},r}-V_1^{\cP,\overline{\pi},r} \leq \epsilon$, after 
$$
H\cdot K=\cO\left(\frac{H^{5}|\mathcal{A}|^{2L} d^{4}\log (d |\cF| / \delta)}{\epsilon^2}\right).
$$
samples, where $\pi^*$ is the optimal policy of $\Pi^{\operatorname{gen}}$.
\end{theorem}

Theorem \ref{theorem:sample} indicates that the sample complexity of PORL$^2$ only depends polynomially on the rank $d$, the history step $L$, horizon $H$, the size of the effective action space for L-memory policy $|\cA|^L$, and the statistical complexity of the function class $\log (|\cF|)$. In particular, Theorem \ref{theorem:sample} avoids direct dependency on the size of the state or observation space. Specifically, we emphasized the term $|\Acal|^{2L}$, which comes from doing importance sampling for $2L$ times in line 3 of Algorithm \ref{PORL_decode}. Our sample complexity also matches the regret bound of \citet{liu2022partially}, Theorem 7.
\section{Low-rank POMDPs with $\gamma$ - Observability}
In this section, we move on to the observability POMDP setting
\begin{assum}[\citet{golowich2022learning,golowich2022planning,even2007value}]\label{observability}
Let $\gamma>0$. For $h \in[H]$, let $\mathbb{O}_h$ be the operator with  $\mathbb{O}_h(\cdot \mid s)$, indexed by states $s$. We say that the operator $\mathbb{O}_h$ satisfies $\gamma$-observability if for each $h$, for any distributions $b, b'$ over states, $\|\mathbb{O}_h b-\mathbb{O}_h b'\|_1 \geq \gamma\|b-b'\|_1$ . A POMDP satisfies $\gamma$-observability if all $h\in [H]$ of satisfy $\gamma$ - observability.
\end{assum}
Assumption \ref{observability} implies that the operator $\mathbb{O}_h:\Delta(S)\rightarrow \Delta(O)$ is an injection.
We use $\tau_i$ to denote $(o_{3-2L:i},~a_{3-2L:i-1})$.
In addition, we make the same realizability assumption (Assumption \ref{ass_realizability_deco}) as in the decodable setting.

 For $\gamma$-observable low-rank POMDPs, we present the assumption that are commonly adopted in the literature to avoid challenges associated with reinforcement learning with function approximation \citep{efroni2022provable}. We state the function approximation and computational oracles below.
\begin{assum}\label{realizability}
There exists a known model class $\cF=\{(\mathbb{O}_h,\omega_h,\psi_h):\mathbb{O}_h\in \Gcal,\omega\in \Omega, \psi_h\in \Psi_h\}_{h=1}^{H}$, where $\mathbb{O}_h^*\in \Gcal,~\omega_h^*\in \Omega$ and $\psi^*_h\in \Psi$ for all $h\in[H]$. 
Recall that $\mathbb{P}_h(o_{h+1} \mid s_h, a_h)=\mathbb{O}^*_{h+1}(o_{h+1}\mid s_{h+1}) \omega_{h+1}^*(s_{h+1})^{\top}\psi_h^*(s_h,a_h)$.
\end{assum}
Compared to Assumption \ref{ass_realizability_deco}, we assume the model class contains the transition information of the latent state in this assumption. 
\subsection{The approximated MDP $\cM$ with $L$-structure}\label{approximated_MDP}
It has been shown that when $\gamma$-observability holds, the POMDP can be approximated by an MDP whose state space is $\cZ=\cO^L\times \cA^{L-1}$ \cite{uehara2022provably}. In particular, for a probability transition kernel $\mathbb{P}_h(z_{h+1} \mid z_h, a_h)$ (for $z_h, z_{h+1} \in \cZ)$ and a reward function $r=(r_1, \ldots, r_H), r_h: \cZ \rightarrow \mathbb{R}$, we will consider MDPs of the form $\mathcal{M}=(\cZ, \mathcal{A}, H, r, \mathbb{P}$). For such an $\mathcal{M}$, we say that $\mathcal{M}$ has $L$-structure if:  the transitions $\mathbb{P}_h(\cdot \mid z_h, a_h)$ have the following property, for $z_h \in \cZ, a_h \in \mathcal{A}$ : writing $z_h=$ $(o_{h-L+1:h}, a_{h-L: h-1}), \mathbb{P}_h(z_{h+1} \mid z_h, a_h)$ is nonzero only for those $z_{h+1}$ of the form $z_{h+1}=( o_{h-L+2: h+1},a_{h-L+1: h})$, where $o_{h+1} \in \mathcal{O}$. 

For a low-rank POMDP $\cP$, $o_{h+1}$ can only be predicted by the whole memory $\{o_1,~a_1,\cdots,~o_{h},a_h\}$. The main observation in \cite{uehara2022provably} is that $o_{h+1}$ can instead be approximated predicted by the $L$-memory $\{o_{h+1-L},~a_{h+1-L},\cdots,~o_{h},a_h\}$ with an error bound $\epsilon_1$, given the memory length is at least $L=  O(\gamma^{-4} \log (d / \epsilon_1))$. In other words, there exists an approximated \textbf{low-rank MDP $\cM$ with $L$-structure} that is close to $\cP$.
For any $\cP=(\mathbb{O},\omega,\psi)$, we can construct an approximated MDP $\cM=\{(\mu_h,\phi_h)\}_{h=1}^H$, where $(\phi,\mu)=q(\mathbb{O},\omega,\psi)$ for an explicit function $q$. The analytical form of $q$ is not important for our discussion and is deferred to Appendix \ref{construct_q_function}. This approximated MDP $\cM$ satisfies that
\#\label{eq:M_construct}
\mathbb{P}_h^{\cP}(o_{h+1}\mid z_h,a_h)=\mu_{h+1}^{\top}(o_{h+1})\phi_h(z_h,a_h).
\#
At the same time, the POMDP $\cP$ satisfies that 
\#\label{eq:P_construct}
\mathbb{P}_h^{\cP}(o_{h+1}\mid \tau_h,a_h)=\mu_{h+1}^{\top}(o_{h+1})\xi_h(\tau_h,a_h),
\#
where the definition of $\xi$ is also defferred to Appendix \ref{construct_q_function}.

%With Assumption \ref{ass_realizability_deco}, we have the model class $\cM=\{(\mu,\phi):\mu\in \Theta,\phi\in \Phi\}$, where $(\mu,\phi)\in \cM$ is induced from $(\mathbb{O},\omega,\phi)\in \cF$ by \eqref{eq:POMDP_realize}.

The constructed $\cM$ retains the structure of low-rank POMDP, and we have the following proposition:
\begin{proposition}\label{prop:app-exist}
For any $\epsilon_1>0$, there exists an L-structured MDP $\cM$ with $L=O(\gamma^{-4} \log (d / \epsilon_1))$, such that for all $\pi \in \Pi^{gen}$ and $h\in [H]$,
\$&
\mathbb{E}_{a_{1: h}, o_{2: h} \sim \pi}\|\mathbb{P}_{h}^{\mathcal{M}}(o_{h+1} \mid z_{h}, a_{h})-\mathbb{P}_{h}^{\mathcal{P}}(o_{h+1} \mid o_{1: h}, a_{1: h})\|_1 \\&\qquad\quad\leq \epsilon_1.
\$
\end{proposition}

By Proposition \ref{prop:app-exist}, 
we have that the conditional probability $\mathbb{P}^{\cP}(o \mid z, a)$ is approximately low rank. Now, we can define the value function under $\cM$ as $V^{\pi,\cM,r}(o_1)=\mathbb{E}^{\pi,\cM}_h[\sum_{i=h}^Hr_i\mid o_1]$. With Proposition \ref{prop:app-exist}, we can prove that for a $L$-memory policy $\pi$, the value function of $\pi$ in $\cM$ can effectively approximate the value function under $\cP$.
\begin{lemma}\label{dif}
With $\cM$ defined in \eqref{eq:M_construct}, for any policy $\pi\in \Pi^{\operatorname{gen}}$, we have $$
    |V_{1}^{\pi, \cP, r}(o_1)-V_{1}^{\pi, \cM, r}(o_1)|\leq\frac{H^2\epsilon_1}{2}. 
%\cdot \mathbb{E}_{(o_{1-L: h}, a_{1-L: h-1}) \sim \pi}^{\mathcal{P}}[\sum_{h=1}^{H-1}\|\mathbb{P}_{h}^{\mathcal{P}}(\cdot \mid a_{1: h}, o_{2: h})-\mathbb{P}_{h}^{\mathcal{M}}(\cdot \mid z_{h}, a_{h})\|_{1}]
$$
\end{lemma}
A direct implication of Lemma \ref{dif} is that a near-optimal policy of $\cM$ is also near-optimal in $\cP$.
% , since we have 
% \$
% &V_{1}^{\pi, \cP, r}(o_1)-V_{1}^{\pi^*, \cP, r}(o_1)\\&\leq H^2\epsilon_1+V_{1}^{\pi, \cM, r}(o_1)-V_{1}^{\pi^*, \cM, r}(o_1)\\&\leq H^2\epsilon_1,
% \$
% where $\pi^*$ is the optimal policy of $\Pi^L$. 
Hence we only need to focus on finding the optimal policy in $\Pi^L$, the set of Markovian policy in $\cM$.
%See Appendix \ref{lem:value} for a detailed proof.
%Now we present a proof sketch of Proposition \ref{prop:app-exist}.

%Lemma \ref{masa} shows that $\|\overline{b}_h(o_{h-L:h},a_{h-L:h-1})-b_h(o_{3-2L:h},a_{3-2L:h-1})\|_1$ is small in expectation under any policy $\pi$. Combined with   \eqref{eq:pomdp_low_decom} and \eqref{eq:mdp_low_decom}, we conclude the proof of Proposition \ref{prop:app-exist}.

%In addition, we assume that we have the following block assumption on the approximated MDP $\cM$, which is common in the representation learning literature \citep{DBLP:journals/corr/abs-2202-00063,DBLP:journals/corr/abs-2102-07035}.

%\begin{definition}[Block MDP]\label{block}
%Consider any $h \in[H]$. A Block MDP has an emission distribution o $(\cdot \mid z) \in \Delta(\mathcal{S})$ and a latent state space transition $T_h(z' \mid z, a)$, such that for any $s \in \mathcal{S}, o_h(s \mid z)>0$ for a unique $z \in \mathcal{Z}$ denoted as $\psi_h^{*}(s)$. Together with the ground truth decoder $\psi_h^{*}$, it defines the transitions $\mathbb{P}_h^{\cM}(s' \mid s, a)=\sum_{z' \in \mathcal{Z}} o_h(s' \mid z') T_h(z' \mid \psi_h^{*}(s), a)$.
%\end{definition}
\subsection{Algorithm and Analysis}
PORL$^2$ for $\gamma$-observable POMDPs is presented in Algorithm \ref{POBRIEE gamma}. It is almost identical to Algorithm \ref{PORL_decode} except for the representation learning step, where in the MLE we maximize $\log \mathbb{P}_h^{\cP}(o_{h+1}\mid \tau_h,a_h)$. Note that we have $\mathbb{P}_h^{\cP}(o_{h+1}\mid \tau_h,a_h)=\xi_h(\tau_h,a_h)^{\top}\mu_{h+1}(o_{h+1})$ as shown in \eqref{eq:P_construct}.

Regarding computational cost, similar to Algorithm \ref{PORL_decode}, Algorithm \ref{POBRIEE gamma} only requires a total of $HK$ calls to the MLE oracle and additional polynomial time complexity for executing the LSVI-LLR planner. The main difference lies in the additional transformation $q$ that needs to be performed when calculating the log-likelihood $\log \mathbb{P}_h^{\cP}(o_{h+1}\mid \tau_h,a_h)$. For instance, when using gradient-based optimization to solve the MLE problem, the total number of such calculations performed is usually small and does not scale directly with the size of the function class.

% We also remark that similar to Algorithm \ref{PORL_decode}, Algorithm \ref{POBRIEE gamma} only needs $O(Hd^2)$ flops in calculating matrix inverse if we use Sherman-Morrison formula. Again, the time complexity is dominated by the greedy step in LSVI-LLR which causes a total $O(AHd^2T^2)$ running time.

\begin{algorithm}[t]
    \begin{algorithmic}[1]
    \caption{Partially Observable Representation Learning for $\gamma$-observable POMDPs (PORL$^2$-$\gamma$-observable)}
    \label{POBRIEE gamma}
        \REQUIRE Representation classes $\{\cF_h\}_{h=0}^{H-1}$, parameters $K$, $\alpha_k$, $\lambda_k$
        \STATE Initialize policy $\pi^0=\{\pi_0, \ldots, \pi_{H-1}\}$ to be arbitrary policies and replay buffers $\cD_h=\varnothing,~\cD'_h=\varnothing$ for all $h$.
        \FOR{$ k \in[K]$} 
            \STATE{Data collection from $\pi^{k-1},~\forall h\in [H],$
            $
            \tau_h\sim d^{\pi_h^{k-1}\circ_{L} U(\cA)}_{\cP};~\cD_h=\cD_h\cup
            \{\tau_h\}$,~
            $\tilde{\tau}_h\sim d^{\pi_h^{k-1}\circ_{2L}U(\cA)}_{\cP};~\cD'_h=\cD'_h\cup\{\tilde{\tau}_h\}.
            $
            %&o_{2-L:h-L-1}:a_{2-L:h-L}\sim d_{\cP}^{\pi^{k-1}},~a_{h-L+1:h}\sim U(\cA),~o_{i+1}\sim \mathbb{P}_h^{\cP}(a_{i-L+1:i}:o_{i-L+1:i})\\&~\text{for}~i\in [h-L-1,h-1];~~~\cD_h=\cD_h\cup\{a_{h-L+1:h}:o_{h-L+1:h+1}\}\\& \tilde{o}_{2-L:h-L-2}:\tilde{a}_{2-L:h-2L}\sim d_{\cP}^{\pi^{k-1}},~\tilde{a}_{h-2L+1:h}\sim U(\cA),~\tilde{o}_{i+1}\sim \mathbb{P}_h^{\cP}(\tilde{a}_{i-L+1:i}:\tilde{o}_{i-L+1:i})\\&~\text{for}~i\in[h-L-2,h-1];~~~\cD'_h=\cD'_h\cup\{\tilde{a}_{h-L+1:h}:\tilde{o}_{h-L+1:h+1}\}\$
            }
           
            \STATE
            Learn representations for all $h\in [H]:(\mathbb{O}^k,\omega^k,\psi^k)=\argmax_{(\mathbb{O}_h,\omega_h,\psi_h)\in \cF}$ $\mathbb{E}_{\cD\cup\cD'}[\log\xi_h(\tau_h,a_h)^{\top} \mu_{h+1}(o_{h+1})]$ by \eqref{eq:P_construct}.
            \STATE{Learn the $L$-step feature: $(\widehat{\phi}^k,\widehat{\mu}^k)=q(\mathbb{O}^k,\omega^k,\psi^k)$}
            \STATE{ Define exploration bonus for all $h\in [H]$:
            $\widehat{b}_h^k(z,a)=\min\bigg\{\alpha_k\sqrt{\widehat{\phi}_h^k(z,a)^{\top}\Sigma_h^{-1}\widehat{\phi}^k_h(z,a)},2\bigg\},
            $
            with $\Sigma_h:=\sum_{z\sim\cD_h}\widehat{\phi}_h^k(z,a)\widehat{\phi}_h^k(z,a)^{\top}+\lambda_{k}I$.
            }
            %\STATE{(for $\gamma$-observable low-rank POMDP) Define exploration bonus for all $h\in [H]$:
            %\$\widehat{b}_h^k(z,a)=\min\bigg\{\alpha_k\sqrt{\widehat{\phi}_h^k(z,a)^{\top}\Sigma_h^{-1}\widehat{\phi}^k_h(z,a)},2\bigg\},\$with $\Sigma_h:=\sum_{z\sim\cD_h}\widehat{\phi}_h^k(z,a)\widehat{\phi}_h^k(z,a)^{\top}+\lambda_{k}I$.}
            \STATE{Set $\pi^k$ as the policy returned by:
            $
            \operatorname{LSVI-LLR}(\{r_{h}+\widehat{b}_{h}^{k}\}_{h=0}^{H-1},\{\widehat{\phi}_{h}^{k}\}_{h=0}^{H-1},\{\widehat{\mu}_{h}^{k}\}_{h=0}^{H-1},\{\mathcal{D}_{h} \cup \mathcal{D}_{h}'\}_{h=0}^{H-1}, \lambda_{k}).
            $
            }
        \ENDFOR
        \STATE{\textbf{Return} $\pi^0,\cdots,\pi^K$}
    \end{algorithmic}
\end{algorithm}

The next theorem shows that Algorithm \ref{POBRIEE gamma} achieves the same sample complexity in $\gamma$-observable POMDPs as in Theorem \ref{theorem:sample}.
\begin{theorem}[Pac Bound of PORL$^2$-$\gamma$- observability]\label{theorem:sample_POMDP}
Under Assumption \ref{observability} and Assumption \ref{ass_realizability_deco}, Let $\delta, \epsilon \in(0,1)$ be given, and let $\overline{\pi}$ be a uniform mixture of $\pi^0, \ldots, \pi^{K-1}$. By setting the parameters as
$$
\begin{aligned}
&\alpha_k=\tilde{\Theta}(\sqrt{k|\mathcal{A}|^{L} \zeta_{k}+\lambda_{k} d}),~\lambda_k=\Theta(d \log (|\cF| k/\delta)), \\
& \epsilon_1=\Theta(\epsilon/(H^2 d^{1/2}\gamma^{-4}\log(1/\epsilon)\log (dHA|\cF|/\delta)^{1/2})),\\&L= \Theta( \gamma^{-4} \log (d / \epsilon_1)),~\zeta_k=\Theta(\log(|\cF|k/\delta)/k),
\end{aligned}
$$

with probability at least $1-\delta$, we have $V^{\cP,\pi^{*},r}-V^{\cP,\overline{\pi},r} \leq \epsilon$, after 
$$
H\cdot K= \cO\bigg(\frac{H^{5}|\mathcal{A}|^{2L} d^{4}\log (d |\cF| / \delta)}{\epsilon^2}\bigg).
$$
% \xz{modify the  sample complexity to only depend on $\epsilon$, not $L$ and $\epsilon_1$.}
episodes of interaction with the environment, where $\pi^*$ is the optimal policy of $\Pi^{\operatorname{gen}}$.
\end{theorem}
We remark on the $|\Acal|^{2L}$ term, which comes from the importance sampling in Algorithm \ref{POBRIEE gamma}, matches the regret bound in \citet{liu2022partially},  and the sample complexity in \cite{golowich2022learning} for $\gamma$-observable POMDP. In addition, it has been shown in \cite{golowich2022planning}, Theorem 6.4 that this sample complexity is necessary for any computational-efficient algorithm.

\section{Highlight of the Analysis}
In this section, we highlight the critical observations in our analysis of Theorem \ref{theorem:sample} and Theorem \ref{theorem:sample_POMDP}.
\paragraph{MLE guarantee.}
The following lemma upper-bounds the reconstruction error with the learned features at any iteration of PORL$^2$.
\begin{lemma}[MLE guarantee]\label{guarantee}
Set $\lambda_k=\Theta(d \log (|\cF|k/\delta))$, for any time step $h\in[H]$, denote $\rho_h$ as the joint distribution for $(w, a)$ in the dataset $\mathcal{D}$ of step $k$, with probability at least $1-\delta$ we have 
\begin{align}
&\mathbb{E}_{w, a \sim \rho}[\|\mathbb{P}_h^{\widehat{\cP}}(\cdot\mid w,a)^{\top}-\mathbb{P}_h^{\cP}(\cdot\mid w,a)^{\top}\|_1^2] \leq \zeta_k,
\end{align}
recall that $\zeta_k=O(\log(|\cF|k/\delta)/k)$.
Here $w$ is the trajectory $\tau$ in the decodable case (Algorithm \ref{PORL_decode}), and $w$ is the state $z$ of the approximated MDP $\Mcal$ in the $\gamma$-observable case (Algorithm \ref{POBRIEE gamma}).
\end{lemma}
% Our MLE procedure for $\gamma$-observable POMDPs differs to the MLE procedure of REP-UCB \citep{uehara2021representation}, since the 
% \zl{not complete}
\paragraph{$L$-step Back Inequality.}
It can be observed that the MLE guarantee ensures the expectation of model estimation error scales as $\tilde{O}(1/k)$ scale under the occupancy distribution of the average policy of $\pi^0,...,\pi^{k-1}$. However, to estimate the performance of the policy $\pi^k$, we must perform a distribution transfer from the distribution induced by $\pi^0,...,\pi^{k-1}$ to the distribution induced by $\pi^k$. To deal with this issue, we generalize the one-step-back technique in \citet{uehara2021representation,agarwal2020flambe} and propose a novel $L$-step-back inequality to handle L-memory policies. This $L$-step back inequality,i.e., moving from $h$ to $h-L$, leverages the bilinear structure in $\mathbb{P}^{\cP}_{h-L}$. 

\begin{lemma}[$L$-step back inequality]
For any $h\in[H]$ and $k\in[K]$, consider function $g_{h}$ that satisfies $g_{h} \in \mathcal{Z}_h \times \mathcal{A}_h \rightarrow \mathbb{R}$, s.t. $\|g_{h}\|_{\infty} \leq B$. Then, for any policy $\pi$, we have 
\$
& \mathbb{E}^{\cP}_{\pi}[g(z_h,a_h)]\\& \leq \mathbb{E}^{\cP}_{(z,a)_{h-L-1}\sim \pi}\bigg[\|\phi(z_{h-L-1}, a_{h-L-1})\|_{\Sigma_{\beta^k_{h-L-1},\phi_{h-L-1}}}^{-1}\bigg]\\& \qquad\quad \cdot \sqrt{|A|^L k\cdot\mathbb{E}_{(\tilde{z}_{h},\tilde{a}_{h}) \sim \gamma_{h}}\{[g(\tilde{z}_{h},\tilde{a}_{h})]^{2}\}+B^{2} \lambda_{k} d},
\$
where we denote $\gamma^k_h=1/k\sum_{i=0}^{k-1}d^{\pi^i}_h(z,a)$ and $\beta^k_h=1/k\sum_{i=0}^{k-1}d^{\pi^i\circ_{2L}U(A)}_h(z,a)$ as the mixture state-action distribution, and $\Sigma_{\beta^k_{h-L-1},\phi}=k\mathbb{E}_{z,a\sim \beta_{h-L-1}}\phi_{h-L-1}(z,a)\phi_{h-L-1}(z,a)^{\top}+\lambda_k I$ as the regularized covariance matrix under the representation $\phi$. 
\end{lemma}

\paragraph{Almost Optimism.}
For $\pi^k= \text{LSVI-LLR}(r+\widehat{b},\widehat{\phi},\{\cD^k\cup \cD^{k \prime},\lambda_k\})$, we can prove that the value function of $\pi^k$ is almost optimism at the initial state distribution.
\begin{lemma}[Almost Optimism at the Initial State Distribution]\label{lem:almost opt}  Using the parameters of Theorem \ref{theorem:sample_POMDP}, with probability $1-\delta$, we have for all iterations $k\in[K]$,
\$V^{\pi^{*},\widehat{\cP}_k,r+\widehat{b}^k}-V^{\pi^{*},\cP,r} \geq-\sigma_k,
\$
where $\sigma_k=\cO(Ld (A^L\log (d k|\cF| / \delta)/k)^{1/2})$.
\end{lemma}

Then, using Lemma \ref{lem:almost opt} and the simulation lemma, we can establish an upper bound on the policy regret for $\pi^k$, we have:

\#\nonumber
&V_1^{\pi^{*},\cP,r}-V_1^{\pi^{k},\cP,r}
\\\nonumber&\quad\leq V^{\pi^k,\widehat{\cP},r+\widehat{b}^k}-V^{\pi^{k},\cP,r}+\sigma_k\\\nonumber&\quad\leq \sum_{h=0}^{H-1}\mathbb{E}_{z,a}\biggl[\widehat{b}_h^k+\mathbb{E}_{o'\sim \mathbb{P}^{\cP}_h(z,a)}V_{h+1}^{\pi^k,\cP,\widehat{b}_h^k}(z_{h+1})\\\label{eq:optimism}&\qquad\quad-\mathbb{E}_{o'\sim \mathbb{P}^{\widehat{\cP}}_h(z,a)}V_{h+1}^{\pi^k,\cP,\widehat{b}_h^k}(z_{h+1})\biggl]+\sigma_k,
\#
where $z_{h+1}=c(z_h,a_h,o')$. 

Finally, we give an upper bound for  \eqref{eq:optimism}. We adopt the idea of the moment-matching policy in \citep{DBLP:journals/corr/abs-2202-03983,https://doi.org/10.48550/arxiv.2206.12020}, which analyze the latent state of the past $L$-steps. For any $g$, $\mathbb{E}^{\cP}_{\pi}[g(z_h,a_h)]$ can be written in a bilinear form
\$
&\biggl\langle \mathbb{E}^{\cP}_{z_{h-L},a_{h-L}\sim \pi}\phi^{\top}(z_{h-L}, a_{h-L}),\int_{s_{h-L+1} \in \mathcal{S}} \omega(s_{h-L+1})\\&\qquad\cdot \mathbb{E}^{\cP}_{a_{h-L+1: h} \sim \mu^{\pi, h}}[g(z_h,a_h) \mid s_{h-L+1}]\operatorname{ds}_{h-L+1}  \biggl\rangle,
\$
where $\mu$ is the moment matching policy defined in \eqref{eq:moment}.

Now we denote $\gamma_h^k=1/k\sum_{i=0}^{k-1}d^{\pi^i}_h(z,a)$ to represent the mixture state-action distribution. Additionally, we define $\Sigma_{\gamma_h^k,\phi^*}=k\mathbb{E}_{z,a\sim \gamma_h^k}\phi^*(z,a)\phi^*(z,a)^{\top}+\lambda_k I$ as the regularized covariance matrix under the ground truth representation $\phi^*$. To derive an upper bound for the aforementioned bilinear form, we can apply the Cauchy-Schwartz inequality within the norm induced by $\Sigma_{\gamma_h^k,\phi^*}$ as follows:
\#\nonumber
&\biggl\langle \mathbb{E}^{\cP}_{z_{h-L},a_{h-L}\sim \pi}\phi_h^{\top}(z_{h-L}, a_{h-L}),\int_{s_{h-L+1} \in \mathcal{S}} \omega_h(s_{h-L+1})\\\nonumber&\qquad\cdot \mathbb{E}^{\cP}_{a_{h-L+1: h} \sim \mu^{\pi, h}}[g(z_h,a_h) \mid s_{h-L+1}]\operatorname{ds}_{h-L+1}  \biggl\rangle\\\nonumber&
\leq \mathbb{E}^{\widehat{\cP}}_{z_{h-L},a_{h-L}\sim \pi}\|\widehat{\phi}_h^{\top}(z_{h-L}, a_{h-L})\|_{\sum_{\gamma_{h-L},\widehat{\phi}_{h-L}}^{-1}}\\\nonumber&\qquad\cdot \bigg\|\int_{s_{h-L+1} \in \mathcal{S}} \widehat{\omega}_h(s_{h-L+1})\cdot\mathbb{E}^{\widehat{\cP}}_{a_{h-L+1: h} \sim \mu^{\pi, h}}[g(z_h,a_h)  \\\label{eq:optimism2}& \qquad\qquad \mid s_{h-L+1}]\operatorname{ds}_{h-L+1}\bigg\|_{\sum_{\gamma_{h-L},\widehat{\phi}_{h-L}}},
\#
The first term in equation \eqref{eq:optimism2} is associated with the elliptical potential function. By employing the $L$-step back inequality, we can transform the second term in equation \eqref{eq:optimism2} into an expectation over the dataset distribution. This expectation can be controlled by leveraging the MLE guarantee. 
\paragraph{Approximated Transition Error}
For the low-rank POMDPs with $\gamma$-observability, the main idea is to analyze the value function under the approximated MDP $\cM$ instead of the real POMDP transition $\cP$. We remark that is a novel technique in low-rank POMDPs with $\gamma$-observable assumption, which analyzes in a Markovian model with a small approximation error. The detailed proof of Theorem \ref{theorem:sample_POMDP} can be found in Appendix \ref{lem:value}.
%We have the following lemma to illustrate the approximated transition error.
\section{Experiments}
\begin{figure}[t]
  \centering
  
    \includegraphics[width=0.45\textwidth]{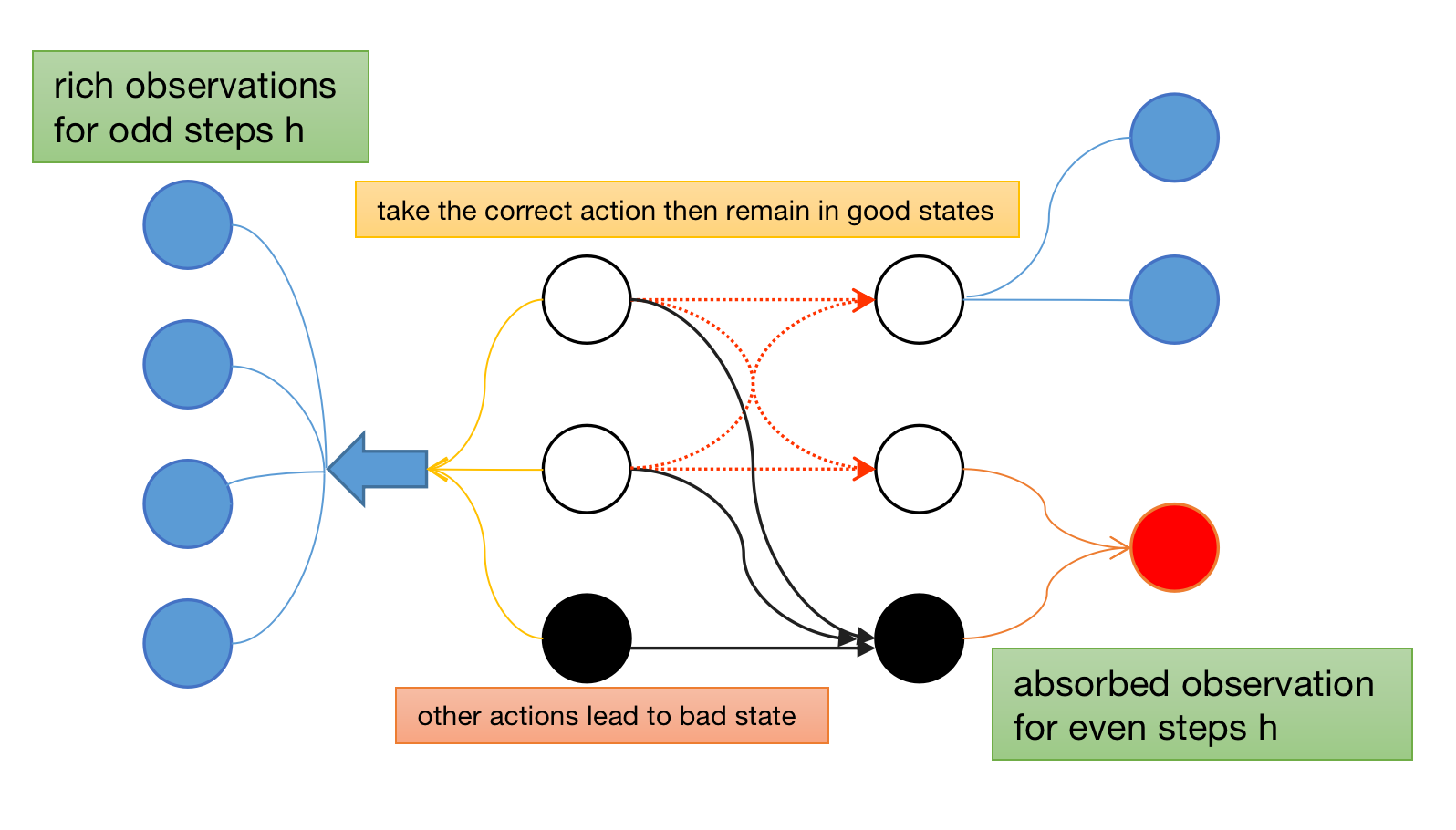}
    \caption{ Visualization of pocomblock, where the blue area represents the rich observations obtained from the latent states. The black arrows illustrate the transition from the good states (depicted in white) to the bad state (depicted in black). Conversely, the red arrows indicate remaining in the good states by taking the correct action. Once the agent transitions to the bad state, it remains in that state for the entire episode, resulting in a failure to achieve the goal. Moreover, when the value of step $h$ is even, there exists an absorbed observation that encompasses all the observations for both the bad state and one of the good states. Consequently, the agent is unable to distinguish between these two states solely based on a one-step observation during such time steps. Thus the name partially observed combination lock, (see Appendix \ref{sed:details} for details).}
    \label{fig:subfig1}
\end{figure}
We evaluate the performance of PORL$^2$ using the partially observed combination lock (pocomblock) as our benchmark, which is inspired by the combination lock benchmark introduced by \citet{misra2019kinematic}. Pocomblock consists of latent states accompanied by rich observations. Further details and specifics regarding the experiments can be found in Appendix \ref{sed:details}.

Next, we provide an overview of pocomblock. Pocomblock has three states: two good states and one bad state. When the agent is in a good state, it will remain in either of the two good states only if the correct action is taken; otherwise, it will transition to the bad state. Once the agent enters the bad state, it becomes impossible to exit. In the bad state, the agent receives a reward of zero.

Then we explain the emission kernel, when the time step $h$ is odd, all the latent states generate rich observations and different states generate different observations. When $h$ is even, one good state still has a rich observation space, while the other states' observations were absorbed by an absorbed observation. Hence at this time, it is unable to distinguish such two states by the current observation. 
\paragraph{Comparision with BRIEE.}We test PORL$^2$ and BRIEE \citep{DBLP:journals/corr/abs-2202-00063} ---the SOTA block MDP algorithm in pocomblock. We note that pocomblock is not a block MDP since the existence of the absorbed observation when $h$ is even. The details and results can be found in Appendix \ref{sed:details}

\section{Conclusion}
We presented a representation learning-based reinforcement learning algorithm for decodable and observable POMDPs, that achieves a polynomial sample complexity guarantee and is amendable to scalable implementation. Future works include empirically designing an efficient and scalable implementation of our algorithm and performing extensive empirical evaluations on public benchmarks, and theoretically extending our algorithm and framework to handle more general decision-making problems such as Predictive State Representations and beyond.
\paragraph{Acknowledgements.} Mengdi Wang acknowledges the support by NSF grants DMS-1953686, IIS-2107304, CMMI-1653435, ONR grant 1006977, and C3.AI.

\newpage

\bibliography{example_paper}
\newpage

%%%%%%%%%%%%%%%%%%%%%%%%%%%%%%%%%%%%%%%%%%%%%%%%%%%%%%%%%%%%%%%%%%%%%%%%%%%%%%%
%%%%%%%%%%%%%%%%%%%%%%%%%%%%%%%%%%%%%%%%%%%%%%%%%%%%%%%%%%%%%%%%%%%%%%%%%%%%%%%
% APPENDIX
%%%%%%%%%%%%%%%%%%%%%%%%%%%%%%%%%%%%%%%%%%%%%%%%%%%%%%%%%%%%%%%%%%%%%%%%%%%%%%%
%%%%%%%%%%%%%%%%%%%%%%%%%%%%%%%%%%%%%%%%%%%%%%%%%%%%%%%%%%%%%%%%%%%%%%%%%%%%%%%
\newpage
\appendix
\onecolumn

\section{Technical Lemma}\label{technical}
In this section we introduce several lemmas which are useful in our proof.
\begin{lemma}[\citep{uehara2021representation}]\label{lemma14}
Consider the following process. For $n=1, \cdots, N, M_n=M_{n-1}+G_n$ with $M_0=\lambda_0 I$ and $G_n$ being a positive semidefinite matrix with eigenvalues upper-bounded by $1$ . We have that:
$$
2 \log \operatorname{det}(M_N)-2 \log \operatorname{det}(\lambda_0 I) \geq \sum_{n=1}^N \operatorname{Tr}(G_n M_{n-1}^{-1}) .
$$
\end{lemma}
\begin{lemma}[Elliptical Potential Lemma, \citep{uehara2021representation}]\label{potential}
Suppose $\operatorname{Tr}(G_n) \leq B^2$, where $G_n$ being a positive semidefinite matrix with eigenvalues upper-bounded by $1$.
$$
2 \log \operatorname{det}(M_N)-2 \log \operatorname{det}(\lambda_0 I) \leq d \log (1+\frac{N B^2}{d \lambda_0}).
$$
\end{lemma}
Next, we provide an important lemma to ensure the concentration of the bonus term in our algorithm. This lemma is proved in Lemma 39 of \cite{zanette2021cautiously}.
\begin{lemma}[Concentration of the Bonus]\label{bonus_concentration}
Set $\lambda_n=\Theta(d \log (n|\cF| / \delta))$ for any $n$. Let $\mathcal{D}=\{s_i, a_i\}_{i=0}^{n-1}$ be a stochastic sequence of data where $(s_i, a_i) \sim \rho_i$ where $\rho_i$ can depend on the history of time steps $1, \ldots, i-1$. Let $\rho=\frac{1}{n} \sum_{i=0}^{n-1} \rho$ and define
$$
\Sigma_{\rho, \phi}=k \mathbb{E}_\rho[\phi(s, a) \phi^{\top}(s, a)]+\lambda_n I, \quad \widehat{\Sigma}_{n, \phi}=\sum_{i=0}^{n-1} \phi(s_i, a_i) \phi^{\top}(s_i, a_i)+\lambda_n I .
$$
Then, with probability $1-\delta$, we have
$$
\forall n \in \mathbb{N}^{+}, \forall \phi \in \Phi, c_1\|\phi(s, a)\|_{\Sigma_{\rho, \phi}^{-1}} \leq\|\phi(s, a)\|_{\widehat{\Sigma}_{n, \phi}^{-1}} \leq c_2\|\phi(s, a)\|_{\Sigma_{\rho, \phi}^{-1}}.
$$
\end{lemma}
We also introduce the Simulation Lemma which is frequently used in RL literature. Its proof can be found in \cite{uehara2021representation}.
\begin{lemma}[Simulation Lemma]\label{simulation}
Given two MDPs $(P', r+b)$ and $(P, r)$, for any policy $\pi$, we have:
$$
V_{P', r+b}^\pi-V_{P, r}^\pi=\sum_{h=1}^H \mathbb{E}_{(s_h, a_h) \sim d_{P', h}^\pi}[b_h(s_h, a_h)+\mathbb{E}_{P_h'(s_h' \mid s_h, a_h)}[V_{P, r, h+1}^\pi(s_h')]-\mathbb{E}_{P_h(s_h' \mid s_h, a_h)}[V_{P, r, h+1}^\pi(s_h')]],
$$
and
$$
V_{P', r+b}^\pi-V_{P, r}^\pi=\sum_{h=1}^H \mathbb{E}_{(s_h, a_h) \sim d_{P, h}^\pi}[b_h(s_h, a_h)+\mathbb{E}_{P_h'(s_h' \mid s_h, a_h)}[V_{P, r+b, h+1}^\pi(s_h')]-\mathbb{E}_{P_h(s_h' \mid s_h, a_h)}[V_{P, r+b, h+1}^\pi(s_h')]].
$$
\end{lemma}

We note that since both the occupancy measure and Bellman updates under $\widehat{P}$ are defined in the exact same way as if $\widehat{P}$ is a proper probability matrix, the classic simulation lemma also applies to $\widehat{P}$.
\section{Proof of Theorem \ref{theorem:sample}}
First we define a few mixture notations that will be used extensively in the analysis. For any $k$, we define $\overline{\pi}^k$ to be $\sum_{i=0}^{k-1}\pi^k/k$ . 
For a fixed $(k,h)$, $\rho_h^k$ is the distribution on $\Zcal\times\Acal$ induced by applying $\{\bar{\pi}_i^k\}_{i=1}^h$ and then do uniformly random action for $L$ times.
We define the distribution of $(z,a)$
For any $h$, define $\rho_{h}^{k}\in \Delta(\mathcal{Z} \times \mathcal{A})$ as follows:
$$
\rho_{h}^{k}(z, a)=d_{\cP,h}^{\overline{\pi}_h^{k}\circ_{L} U(\cA)} (z,a).
$$
Similarly, we define the distribution on $\Zcal\times\Acal$ after doing random action for $2L$ times. For any $k, h \geq 1$, we define $\beta_{h}^{k}$ as follows:
\$
\beta_{h}^{k}(z, a)= d_{\cP,h}^{\overline{\pi}_h^{k}\circ_{2L} U(\cA)} (z,a).
\$
We also define the distribution induced by $\bar{\pi}^k$. For any $k, h$, we also define $\gamma_{h}^{k} \in \Delta(\mathcal{S} \times \mathcal{A})$ as follows:
$$
\gamma_{h}^{k}(z, a)=d_{\cP,h}^{\overline{\pi}^{k}}(z, a).
$$
For notational simplification, we denote $\|x\|_{\rho,\phi}=\|x\|_{\sum_{\rho,\phi}}$ and $\|x\|_{\rho^{-1},\phi}=\|x\|_{\sum^{-1}_{\rho},\phi}$ for $x\in \mathbb{R}^d$ and $\phi\in \Phi$, where $\sum_{\rho,\phi}=\mathbb{E}_{z,a\sim \rho}[\phi(z,a)\phi^{\top}(z,a)]+\lambda I$. We define $\|x\|_{\beta,\phi}$, $\|x\|_{\gamma,\phi}$ in the same way.

By Lemma \ref{guarantee}, with probability at least $1-\delta$, we assume that we have learned a $\widehat{\phi}$ and $\widehat{\mu}$ such that 
\$
& \mathbb{E}_{(z, a) \sim \rho_h}[\|\mathbb{P}_h^{\widehat{\cP}}(\cdot \mid z, a)^{\top} -\mathbb{P}_h^{\cP}(\cdot \mid z, a)^{\top}\|_1^2] \leq \zeta_k, \forall h \in[H] . \\
&\mathbb{E}_{(z, a) \sim \beta_h}[\|\mathbb{P}_h^{\widehat{\cP}}(\cdot \mid z, a)^{\top}-\mathbb{P}_h^{\cP}(\cdot \mid z, a)^{\top}\|_1^2] \leq \zeta_k, \forall h \in[H].
\$
Here $\hat{\^P}_h(o'\mid z,a) = \hat{\phi}(z,a)\hat{\mu}(o')$ and $\zeta_k = O(\log(|\Mcal|)/K)$. In later arguments we will condition on these events.
Next, we prove the almost optimism lemma restated below.
\begin{lemma}[Almost Optimism]\label{lem:almost}
Consider episode $k\in[K]$ and set
$$
\alpha_{k}=\sqrt{k|\mathcal{A}|^{L} \zeta_{k}+4 \lambda_{k} d} / c, \quad \lambda_{k}=O(d \log (|\cF| k / \delta)) .
$$
where $c$ is an absolute constant. Then with probability $1-\delta$,
$$
V^{\pi^{*},\widehat{\cP}_k,r+\widehat{b}^k}-V^{\pi^{*},\cP,r} \geq-\frac{c\alpha_kL}{\sqrt{k}} 
$$
holds for all $k \in[K]$.
\end{lemma} 
\begin{proof}
By Lemma \ref{simulation}, we have
\#\nonumber
&V^{\pi^{*},\widehat{\cP},r+\widehat{b}}-V^{\pi^{*},\cP,r} \\\nonumber
&=\sum_{h=0}^{H-1} \mathbb{E}_{(z_{h}, a_{h}) \sim d_{\widehat{\cP},h}^{\pi^{*}}}[\widehat{b}_{h}(z_{h}, a_{h})+\mathbb{E}_{o'\sim\mathbb{P}_{h}^{\widehat{\cP}}(\cdot\mid z_{h}, a_{h})}[V_{ h+1}^{\pi^{*},\cP,r}(z_{h+1}')]-\mathbb{E}_{o'\sim\mathbb{P}_{h}^{\cP}(\cdot\mid z_{h}, a_{h})}[V_{ h+1}^{\pi^{*},\cP,r}(z_{h+1}')]] \\\nonumber
&\geq \sum_{h=0}^{H-1} \mathbb{E}_{(z_{h}, a_{h}) \sim d_{\widehat{P}, h}^{\pi^{*}}}\biggl[\min (c \alpha_{k}\|\widehat{\phi}_{h}(z, a)\|_{\Sigma_{\rho_{h}, \widehat{\phi}_{h}}^{-1}}, 2)+\mathbb{E}_{o'\sim\mathbb{P}_{h}^{\widehat{\cP}}(\cdot\mid z_{h}, a_{h})}[V_{ h+1}^{\pi^{*},\cP,r}(z_{h+1}')]\\ \label{eq:almost1}&\qquad\quad-\mathbb{E}_{o'\sim\mathbb{P}_{h}^{\cP}(\cdot\mid z_{h}, a_{h})}[V_{ h+1}^{\pi^{*},\cP,r}(z_{h+1}')]\biggl],
\#

where in the last step, we replace empirical covariance by population covariance by Lemma \ref{bonus_concentration}, here $c$ is an absolute constant. Here $(z,a)\sim d_{\Pcal,h}^\pi$ means that $(z,a)$ is sampled from transition $\Pcal$ and policy $\pi$. We define
$$
g_{h}(z, a)=\mathbb{E}_{o_h'\sim\mathbb{P}_{h}^{\widehat{\cP}}(\cdot\mid z, a)}[V_{ h+1}^{\pi^{*},\cP,r}(c(z,a,o'_h)
)]-\mathbb{E}_{o_h'\sim \mathbb{P}_{h}^{\cP}(\cdot\mid z, a)}[V_{ h+1}^{\pi^{*},\cP,r}(c(z,a,o'_h))].
$$
Notice that we have $\|g_{h}\|_{\infty} \leq 1$.
By Lemma \ref{guarantee}, for any $(z,a)$ we have 
\$
\mathbb{E}_{(z, a) \sim \rho_{h}}[g_{h}^{2}(z,a)] \leq \zeta_{k},~
\mathbb{E}_{(z, a) \sim \beta_{h}}[g_{h}^{2}(z,a)] \leq \zeta_{k}.
\$
By Lemma \ref{$L$-step learn}, we have:
\#\nonumber&\sum_{h=0}^{H-1} \mathbb{E}_{(z, a) \sim d_{\widehat{\cP}, h}^{\pi^{*}}}[g_{h}(z, a)] \\\leq \nonumber&\sum_{h=1}^H \mathbb{E}_{(z_{h-L-1},a_{h-L-1})\sim d^\pi_{\hat{\Pcal},h-L-1}}\|\widehat{\phi}^{\top}(z_{h-L-1}, a_{h-L-1})\|_{\rho_{h-L-1}^{-1},\widehat{\phi}_{h-L-1}}\\\nonumber& \qquad\quad \cdot \sqrt{|A|^{L} k\cdot\mathbb{E}_{(\tilde{z}_{h},\tilde{a}_{h}) \sim \beta_{h}}\{[g(\tilde{z}_{h},\tilde{a}_{h})]^{2}\}+B^{2} \lambda_{k} d}\\&
\leq \nonumber\sum_{h=1}^H \mathbb{E}_{(z_{h-L-1},a_{h-L-1})\sim d^\pi_{\hat{\Pcal},h-L-1}}\|\widehat{\phi}^{\top}(z_{h-L-1}, a_{h-L-1})\|_{\rho_{h-L-1}^{-1},\widehat{\phi}_{h-L-1}}\\\nonumber& \qquad\quad \cdot \sqrt{|A|^{L} k\zeta_k+B^{2} \lambda_{k} d}
\\&\label{eq:almost2}= c\alpha_k\sum_{h=1}^H \mathbb{E}_{(z_{h-L-1},a_{h-L-1})\sim d^\pi_{\hat{\Pcal},h-L-1}}\|\widehat{\phi}^{\top}(z_{h-L-1}, a_{h-L-1})\|_{\rho_{h-L-1}^{-1},\widehat{\phi}_{h-L-1}},\#
where in the last step we define
$$
\alpha_{k}=\sqrt{k|\mathcal{A}|^{L} \zeta_{k}+4 \lambda_{k} d} / c.
$$
For $h\leq 0$, we have 
\#\label{eq:h<0}
\|\widehat{\phi}_h^{\top}(z_h,a_h)\|_{\rho_h^{-1},\widehat{\phi}_h}=\sqrt{\frac{1}{k+\lambda}}<\frac{1}{\sqrt{k}}.
\#
Combine \eqref{eq:almost1}, \eqref{eq:almost2} and \eqref{eq:h<0}, we conclude the proof.
\end{proof} 
With Lemma \ref{lem:almost}, we prove that under MDP $\cP$ we can effectively learn the optimal policy with Algorithm \ref{PORL_decode}.
\begin{theorem}\label{theorem:M-regret}
With probability $1-\delta$, we have
$$
 \sum_{k=1}^KV^{\pi^{*},\cP,r}-V^{\pi^k,\cP,r} \leq O(H^{2}|\mathcal{A}|^{L} d^{2} K^{1/ 2} \log (d K|\cF| / \delta)^{1 / 2}).
$$
\end{theorem}
\begin{proof} 
Similar to Lemma \ref{lem:almost}, we condition on the event that the MLE guarantee \ref{guarantee} holds, which happens with probability $1-\delta$.
For fixed  $k$ we have
\$ 
V^{\pi^{*},\cP,r}-V^{\pi^{k},\cP,r}
&\leq V^{\pi^*,\widehat{\cP},r+\widehat{b}^k}-V^{\pi^{k},\cP,r}+\frac{c\alpha_kL}{\sqrt{k}} \\&\leq V^{\pi^k,\widehat{\cP},r+\widehat{b}^k}-V^{\pi^{k},\cP,r}+\frac{c\alpha_kL}{\sqrt{k}},\$
where the first inequality comes from Lemma \ref{lem:almost}, the second inequality comes from $\pi^k=\operatorname{argmax}_{\pi} V^{\pi,\widehat{\cP},r+\widehat{b}^k}$.
By Lemma \ref{simulation}, we have
\$&V^{\pi^k,\widehat{\cP},r+\widehat{b}^k}-V^{\pi^{k},\cP,r}+\frac{c\alpha_kL}{\sqrt{k}}\\
& =\sum_{h=0}^{H-1}\bigg[ \mathbb{E}_{(z_{h}, a_{h}) \sim d_{h}^{\pi^{k},\cP}}\big[\widehat{b}_{h}(z_{h}, a_{h})+\mathbb{E}_{o_h'\sim\mathbb{P}_{h}^{\widehat{\cP}}(o_{h}' \mid z_{h}, a_{h})}[V_{h+1}^{\pi^{k},\widehat{\cP},r+\widehat{b}^k}(z_{h+1}')]\\ &\qquad\quad-\mathbb{E}_{o_h'\sim \mathbb{P}_{h}^{\cP}(o_{h}' \mid z_{h}, a_{h})}[V_{h+1}^{\pi^{k},\widehat{\cP},r+\widehat{b}^k}(z_{h+1}')]\big]\bigg]+\frac{c\alpha_kL}{\sqrt{k}},
\$
 Further we have
\begin{align*}
V^{\pi^k,\widehat{\cP},r+\widehat{b}^k}-V^{\pi^{k},\cP,r}
& =\sum_{h=0}^{H-1}\bigg[ \mathbb{E}_{(z_{h}, a_{h}) \sim d_{h}^{\pi^{k},\cP}}\big[\widehat{b}_{h}(z_{h}, a_{h})+\mathbb{E}_{o_h'\sim\mathbb{P}_{h}^{\widehat{\cP}}(o_{h}' \mid z_{h}, a_{h})}[V_{h+1}^{\pi^{k},\widehat{\cP},r+\widehat{b}^k}(z_{h+1}')]\\ &\qquad\quad-\mathbb{E}_{o_h'\sim \mathbb{P}_{h}^{\cP}(o_{h}' \mid z_{h}, a_{h})}[V_{h+1}^{\pi^{k},\widehat{\cP},r+\widehat{b}^k}(z_{h+1}')]\big]\bigg]
\end{align*}
and
the last equation comes from Lemma \ref{simulation}.
Therefore we have \begin{align*}
V^{\pi^{*},\cP,r}-V^{\pi^{k},\cP,r}
& \leq\sum_{h=0}^{H-1}\bigg[ \mathbb{E}_{(z_{h}, a_{h}) \sim d_{h}^{\pi^{k},\cP}}\big[\widehat{b}_{h}(z_{h}, a_{h})+\mathbb{E}_{o_h'\sim\mathbb{P}_{h}^{\widehat{\cP}}(o_{h}' \mid z_{h}, a_{h})}[V_{h+1}^{\pi^{k},\widehat{\cP},r+\widehat{b}^k}(z_{h+1}')]\\ &\qquad\quad-\mathbb{E}_{o_h'\sim \mathbb{P}_{h}^{\cP}(o_{h}' \mid z_{h}, a_{h})}[V_{h+1}^{\pi^{k},\widehat{\cP},r+\widehat{b}^k}(z_{h+1}')]\big]\bigg]+\frac{c\alpha_kL}{\sqrt{k}}
\end{align*}
Denote
\$
f_h(z_h,a_h)=\frac{1}{2H+1}\biggl(\mathbb{E}_{o_h'\sim\mathbb{P}_{h}^{\widehat{\cP}}(o_{h}' \mid z_{h}, a_{h})}[V_{h+1}^{\pi^{k},\widehat{\cP},r+\widehat{b}^k}(z_{h+1}')]-\mathbb{E}_{o_h'\sim \mathbb{P}_{h}^{\cP}(o_{h}' \mid z_{h}, a_{h})}[V_{h+1}^{\pi^{k},\widehat{\cP},r+\widehat{b}^k}(z_{h+1}')]\biggl).
\$

Note that $\|\widehat{b}\|_{\infty} \leq 2$, hence we have $\|V_{ h+1}^{\pi^k,\widehat{\cP},r+b}\|_{\infty} \leq(2 H+1)$. Combining this fact with the above expansion, we have
\#\nonumber
V^{\pi^{*},\cP,r}-V^{\pi^k,\widehat{\cP},r}&=\sum_{h=0}^{H-1} \mathbb{E}_{(z_{h}, a_{h}) \sim d_{h}^{\pi^k,\cP}}[\widehat{b}_{h}(z_{h}, a_{h})]+(2 H+1) \sum_{h=0}^{H-1} \mathbb{E}_{(z_{h}, a_{h}) \sim d_{ h}^{\pi^k,\cP}}[f_{h}(z_{h}, a_{h})]\\\label{eq:8}&\quad\qquad+\frac{c\alpha_kL}{\sqrt{k}}.
\#
First, we calculate the bonus term in \eqref{eq:8}. , we have
\$&\sum_{h=0}^{H-1} \mathbb{E}_{(z_{h}, a_{h}) \sim d_{h}^{\pi^{k},\cP}}[\widehat{b}_{h}(z_{h}, a_{h})]\\&\leq \sum_{h=0}^{H-1} \mathbb{E}_{(\tilde{z}, \tilde{a}) \sim d_{h-L}^{\pi^k,\cP}}\|\phi_{h-L}^{*}(\tilde{z}, \tilde{a})\|_{\Sigma_{\gamma_{h-L}, \phi_{h-L}^{*}}^{-1}} \sqrt{k|\mathcal{A}|^L \mathbb{E}_{(z, a) \sim \rho_{h}}[(\widehat{b}_{h}(z, a))^{2}]+4 \lambda_{k}} d
%+\sqrt{|\mathcal{A}|^L \mathbb{E}_{(z, a) \sim \rho_{0}}[(\widehat{b}_{0}(z, a))^{2}]} 
.\$
where the inequality is following Lemma \ref{$L$-step} associate with $\|\widehat{b}_{h}\|_{\infty} \leq 2$. 

Note that we use the fact that $B=2$ when applying Lemma \ref{$L$-step}. In addition, we have that for any $h \in[H]$,
$$
k \mathbb{E}_{(z, a) \sim \rho_{h}}\bigg[\|\widehat{\phi}_{h}(z, a)\|_{\Sigma_{\rho_{h}, \widehat{\phi}_{h}}^{-1}}^{2}\bigg]=k \operatorname{Tr}\bigg(\mathbb{E}_{\rho_{h}}[\widehat{\phi}_{h} \widehat{\phi}_{h}^{\top}]\{k \mathbb{E}_{\rho_{h}}[\widehat{\phi}_{h} \widehat{\phi}_{h}^{\top}]+\lambda_{k} I\}^{-1}\bigg) \leq d .
$$
Then we have
$$
\sum_{h=1}^{H} \mathbb{E}_{(z, a) \sim d_{h}^{\pi^k,\cP}}[b_{h}(z, a)] \leq \sum_{h=1}^{H} \mathbb{E}_{(\tilde{z}, \tilde{a}) \sim d_{h-L}^{\pi^k,\cP}}\|\phi_{h-L}^{*}(\tilde{z}, \tilde{a})\|_{\Sigma_{\rho_{h-L}, \phi_{h-L}^{*}}^{-1}} \sqrt{|\mathcal{A}|^L \alpha_{k}^{2} d+4 \lambda_{k} d}.
%+\sqrt{|\mathcal{A}|^L \alpha_{1}^{2} d / k}
$$

Now we bound the second term in  \eqref{eq:8}. Further, with $\|f_{h}(z, a)\|_{\infty} \leq 1$, we have
\$&\sum_{h=0}^{H-1} \mathbb{E}_{(z_{h}, a_{h}) \sim d_{h}^{\pi^k,\cP}}[f_{h}(z_{h}, a_{h})]\\&\leq \sum_{h=0}^{H-1} \mathbb{E}_{(\tilde{z}, \tilde{a}) \sim d_{h-L}^{\pi^k,\cP}}\|\phi_{h-L}^{*}(\tilde{z}, \tilde{a})\|_{\Sigma_{\gamma_{h-L}, \phi_{h-L}^{*}}^{-1}} \sqrt{k|\mathcal{A}|^L \mathbb{E}_{(z, a) \sim \rho_{h}}[f_{h}^{2}(z, a)]+4 \lambda_{k} d}
%+\sqrt{|\mathcal{A}| \mathbb{E}_{(z, a) \sim \rho_{0}}[f_{0}^{2}(z, a)]} 
\\ &\leq \sum_{h=0}^{H-1} \mathbb{E}_{(\tilde{z}, \tilde{a}) \sim d_{h-L}^{\pi^k,\cP}}\|\phi_{h-L}^{*}(\tilde{z}, \tilde{a})\|_{\Sigma_{\gamma_{h-L}, \phi_{h-L}^{*}}^{-1}} \sqrt{k|\mathcal{A}|^L \zeta_{k}+4 \lambda_{k} d},
%+\sqrt{|\mathcal{A}| \zeta_{k}}
\$
where the first inequality is by By Lemma \ref{$L$-step} and in the second inequality, we use $\mathbb{E}_{z, a \sim \rho_{h}}[f_{h}^{2}(z, a)] \leq \zeta_{k}$. Then we have
$$
\begin{aligned}
& V^{\pi^{*},\cP,r}-V^{\pi^{k},\cP,r} \\
=& \sum_{h=0}^{H-1} \mathbb{E}_{(z_{h}, a_{h}) \sim d_{h}^{\pi^k,\cP}}[b_{h}(z_{h}, a_{h})]+(2 H+1) \sum_{h=0}^{H-1} \mathbb{E}_{(z_{h}, a_{h}) \sim d_{h}^{\pi^k,\cP}}[f_{h}(z_{h}, a_{h})]+\frac{c\alpha_kL}{\sqrt{k}} \\
\leq & \sum_{h=0}^{H-1} \mathbb{E}_{(\tilde{z}, \tilde{a}) \sim d_{h-L}^{\pi^{k},\cP}}\|\phi_{h}^{*}(\tilde{z}, \tilde{a})\|_{\Sigma_{\gamma_{h-L}, \phi_{h-L}^{*}}^{-1}} \sqrt{|\mathcal{A}|^L \alpha_{k}^{2} d+4 \lambda_{k} d}
%+\sqrt{|\mathcal{A}|^L \alpha_{1}^{2} d / k}+
\\&+(2 H+1) \sum_{h=0}^{H-1} \mathbb{E}_{(\tilde{z}, \tilde{a}) \sim d_{h-L}^{\pi^{k},\cP}}\|\phi_{h-L}^{*}(\tilde{z}, \tilde{a})\|_{\Sigma_{\gamma_{h-L}^{-1}, \phi_{h-L}^{*}}} \sqrt{k|\mathcal{A}|^L \zeta_{k}+4 \lambda_{k} d}
% +(2 H+1) \sqrt{|\mathcal{A}|^L \zeta_{k}}
+\frac{c\alpha_kL}{\sqrt{k}}.
\end{aligned}
$$
Hereafter, we take the dominating term out. First, recall
$$
\alpha_{k}=O(\sqrt{k|\mathcal{A}|^L \zeta_{k}+4 \lambda_{k} d}).
$$
Second, recall that $\gamma_{h}^{k}(z, a)=\frac{1}{k} \sum_{i=0}^{k-1} d_{h}^{\pi^{i}}(z, a)$. Thus
\$
& \sum_{k=1}^{K}\mathbb{E}_{(\tilde{z}, \tilde{a}) \sim d_{ h}^{\pi^k,\cP}}\|\phi^{*}(\tilde{z}, \tilde{a})\|_{\Sigma_{\gamma_{h}^{n}, \phi_{h}^{*}}^{-1}} \leq \sqrt{K \sum_{k=1}^{K} \mathbb{E}_{(\tilde{z}, \tilde{a}) \sim d_{ h}^{\pi^k,\cP}}[\phi_{h}^{*}(\tilde{z}, \tilde{a})^{\top} \Sigma_{\gamma_{h}^{k}, \phi_{h}^{*}}^{-1} \phi_{h}^{*}(\tilde{z}, \tilde{a})]}\\&\leq \sqrt{K(\log \operatorname{det}(\sum_{k=1}^{K} \mathbb{E}_{(\tilde{z}, \tilde{a}) \sim d_{h}^{\pi^k,\cP}}[\phi_{h}^{*}(\tilde{z}, \tilde{a}) \phi_{h}^{*}(\tilde{z}, \tilde{a})^{\top}])-\log \operatorname{det}(\lambda_{1} I))}\\& \leq \sqrt{d K \log (1+\frac{K}{d \lambda_{1}})},
\$
where the first inequality is by Cauchy-Schwarz inequality, the second inequality is by Lemma \ref{lemma14} and the third inequality is by Lemma \ref{potential}.

Finally, Lemma \ref{guarantee} gives
$$
\zeta_{k}=O\bigg(\frac{\log(|\cF|k/\delta)}{k}\bigg).
%+\epsilon_1^2)
$$
Combining all of the above, we have
$$
 \sum_{k=1}^{K}V^{\pi^{*},\cP,r}-V^{\pi^{k},\cP,r} \leq O(H^{2}|\mathcal{A}|^{L} d^{2} K^{1/ 2} \log (d K|\cF| / \delta)^{1 / 2}),
$$
which concludes the proof.
\end{proof}
\begin{lemma}[$L$-step back inequality for the true model]\label{$L$-step}
Consider a set of functions $\{g_{h}\}_{h=0}^{H}$ that satisfies $g_{h} \in \mathcal{Z} \times \mathcal{A} \rightarrow \mathbb{R}$, s.t. $\|g_{h}\|_{\infty} \leq B$ for all $h \in[H]$. Then, for any policy $\pi$, we have
\$
\sum_{h=1}^H \mathbb{E}^{\cP}_{\pi}[g(z_h,a_h)]& \leq \sum_{h=1}^H \mathbb{E}^{\cP}_{z_{h-L-1},a_{h-L-1}\sim \pi}\bigg[\|\phi(z_{h-L-1}, a_{h-L-1})\|_{\beta^{-1}_{h-L-1},\phi_{h-L-1}}\bigg]\\& \qquad\quad \cdot \sqrt{|A|^L k\cdot\mathbb{E}_{(\tilde{z}_{h},\tilde{a}_{h}) \sim \gamma_{h}}\{[g(\tilde{z}_{h},\tilde{a}_{h})]^{2}\}+B^{2} \lambda_{k} d}.
\$
\end{lemma}
\begin{proof}
For $h\in [H]$ and $h'\in [h-L+1,h]$, we define $\cX_{l}=\cS^{l}\times\cO^{l}\times\cA^{l-1}$ and
\$
x_{h'}=(s_{h-L+1: h'}, o_{h-L+1: h'}, a_{h-L+1: h'-1}),
\$
where $l=h'-h+L-1$.

Now we define the moment matching policy $\mu^{\pi, h}=\{\mu_{h'}^{\pi, h}: \mathcal{X}_l \rightarrow \Delta(\mathcal{A})\}_{h'=h-L+1}^h$. We set $\mu^{\pi,h}$ as following:
\#\label{eq:moment}
\mu_{h'}^{\pi, h}(a_{h'} \mid x_{h'}):=\mathbb{E}^{\cP}_{\pi}[\pi_{h'}(a_{h'} \mid z_{h'}) \mid x_{h'}]~\text{for~}h'\leq h-1,~\text{and}~ \mu_h^{\pi,h}=\pi_h.
\#
Then we define policy $\tilde{\pi}_h$ which takes first $h-L$ actions from $\pi$ and remaining actions from $\mu^{\pi,h}$.

\paragraph{Lemma B.2 in \citet{DBLP:journals/corr/abs-2202-03983}:}  For a fixed $h\in [H]$, any $z_h\in \cZ_h$ and fixed $L$-step policies $\pi$, $d^{\cP,\pi}_h(z_h)=d^{\cP,\tilde{\pi}^h}_h(z_h)$. 

By Lemma B.2 in \citet{DBLP:journals/corr/abs-2202-03983}, we have $d^{\cP,\pi}_h(z_h)=d^{\cP,\tilde{\pi}^h}_h(z_h)$. Then we have $d^{\cP,\pi}_h(z_h,a_h)=d^{\cP,\tilde{\pi}^h}_h(z_h,a_h)$ since $\mu_h^{\pi,h}=\pi_h$. Hence we have $\mathbb{E}^{\cP}_{\pi}[g(z_h,a_h)]=\mathbb{E}^{\cP}_{\tilde{\pi}^h}[g(z_h,a_h)]$.

Since $\mu^{\pi,h}$ is independent of $s_{h-L+1}$, we have the $L$-step-back decomposition:
\$
&\mathbb{E}^{\cP}_{\tilde{\pi}^h}[g(z_h,a_h)]\\ &=\mathbb{E}^{\cP}_{z_{h-L},a_{h-L}\sim \pi}\biggl[\int_{s_{h-L+1}}\phi^{\top}(z_{h-L}, a_{h-L}) \omega(s_{h-L+1}) \cdot \mathbb{E}^{\cP}_{a_{h-L+1: h} \sim \mu^{\pi, h}}[g(z_h,a_h) \mid s_{h-L+1}]\operatorname{ds}_{h-L+1}\biggl]\\
&=\mathbb{E}^{\cP}_{z_{h-L},a_{h-L}\sim \pi}\phi^{\top}(z_{h-L}, a_{h-L})\cdot \int_{s_{h-L+1} \in \mathcal{S}} \omega(s_{h-L+1}) \cdot \mathbb{E}^{\cP}_{a_{h-L+1: h} \sim \mu^{\pi, h}}[g(z_h,a_h) \mid s_{h-L+1}]\operatorname{ds}_{h-L+1}\\&
\leq \mathbb{E}^{\cP}_{z_{h-L},a_{h-L}\sim \pi}\|\phi^{\top}(z_{h-L}, a_{h-L})\|_{\beta^{-1}_{h-L},\phi_{h-L}}\\&\qquad\quad \cdot \bigg\|\int_{s_{h-L+1} \in \mathcal{S}} \omega(s_{h-L+1})\cdot\mathbb{E}^{\cP}_{a_{h-L+1: h} \sim \mu^{\pi, h}}[g(z_h,a_h) \mid s_{h-L+1}]\operatorname{ds}_{h-L+1}\bigg\|_{\beta_{h-L},\phi_{h-L}}.
\$
The first equality comes from the definition of conditional expectation, and the inequality comes from Cauchy-Schwarz inequality. Here we use $\tilde{g}(s_{h-L})$ to denote $\mathbb{E}^{\cP}_{a_{h-L+1: h} \sim \mu^{\pi, h}}[g(z_h,a_h) \mid s_{h-L}]$ for notational simplification. 

We have
\$&\|\int_{s_{h-L+1} \in \mathcal{S}} \omega(s_{h-L+1})\tilde{g}(s_{h-L+1})\operatorname{ds}_{h-L+1}\|^2_{\beta_{h-L},\phi_{h-L}}\\
&=\biggl\{\int_{s_{h-L+1} \in \mathcal{S}} \omega(s_{h-L+1})\tilde{g}(s_{h-L+1})\operatorname{ds}_{h-L+1}\biggl\}^{\top}\\&\qquad\quad\cdot\biggl\{k \mathbb{E}_{(\tilde{z}_{h-L},\tilde{a}_{h-L}) \sim \beta_{h-L+1}}[\phi_{h-L}(\tilde{z}_{h-L},\tilde{a}_{h-L}) \phi^{\top}_{h-L}(\tilde{z}_{h-L},\tilde{a}_{h-L})]+\lambda_{k} I\biggl\}\\&\qquad\quad\cdot\biggl\{\int_{s_{h-L+1} \in \mathcal{S}} \omega(s_{h-L+1})\tilde{g}(s_{h-L+1})\operatorname{ds}_{h-L+1}\biggl\}\\
&\leq k \mathbb{E}_{(\tilde{z}_{h-L},\tilde{a}_{h-L}) \sim \beta_{h-L}}\biggl\{[\int_{s_{h-L+1} \in \mathcal{S}} \omega^{\top}(s_{h-L+1})\phi_{h}(\tilde{z}_{h-L},\tilde{a}_{h-L})\tilde{g}(s_{h-L+1})\operatorname{ds}_{h-L+1}]^{2}\biggl\}+B^{2} \lambda_{k} d\\
&=k \mathbb{E}_{(\tilde{z}_{h-L},\tilde{a}_{h-L}) \sim \beta_{h-L}}\biggl\{\bigg[\mathbb{E}_{s_{h-L+1}\sim \mathbb{P}^{\cP}_{h-L+1}(\cdot\mid\tilde{z}_{h-L},\tilde{a}_{h-L})}
\mathbb{E}_{a_{h-L+1: h} \sim \mu^{\pi, h}}[g(z_{h},a_{h}) \mid s_{h-L+1}]\bigg]^{2}\biggl\}+B^{2} \lambda_{k} d,
\$

where the inequality comes from Cauchy-Schwarz inequality and $\|g_h\|_{\infty}\leq B$. Moreover, we have 
\$
&k \mathbb{E}_{(\tilde{z}_{h-L},\tilde{a}_{h-L}) \sim \beta_{h-L}}\biggl\{[\mathbb{E}_{s_{h-L+1}\sim \mathbb{P}^{\cP}_{h-L+1}(\cdot\mid\tilde{z}_{h-L},\tilde{a}_{h-L})}
\mathbb{E}_{a_{h-L+1: h} \sim \mu^{\pi, h}}[g(z_{h},a_{h}) \mid s_{h-L+1}]]^{2}\biggl\}
\\&\leq  k \mathbb{E}_{(\tilde{z}_{h-L},\tilde{a}_{h-L}) \sim \beta_{h-L},s_{h-L+1}\sim \mathbb{P}^{\cP}_{h-L+1}(\cdot\mid\tilde{z}_{h-L},\tilde{a}_{h-L}),a_{h-L+1:h}\sim \mu^{\pi,h}}[g(z_{h},a_{h})]^{2}
\\
&\leq |A|^L k \mathbb{E}_{(\tilde{z}_{h-L},\tilde{a}_{h-L}) \sim \beta_{h-L},s_{h-L+1}\sim \mathbb{P}^{\cP}_{h-L+1}(\cdot\mid\tilde{z}_{h-L},\tilde{a}_{h-L}),a_{h-L+1:h}\sim U(\cA)}[g(z_{h},a_{h}) ]^{2}
\\
&=|A|^L k\cdot\mathbb{E}_{(\tilde{z}_{h},\tilde{a}_{h}) \sim \gamma_{h}}\{[g(\tilde{z}_{h},\tilde{a}_{h})]^{2}\},
\$

where the first inequality is by Jensen's inequality and the second inequality is by importance sampling, the last equation is by the definition of $\gamma_h$.

Then, the final statement is immediately concluded.
\end{proof}
\begin{lemma}[$L$-step back inequality for the learned model]\label{$L$-step learn}
Consider a set of functions $\{g_{h}\}_{h=0}^{H}$ that satisfies $g_{h} \in \mathcal{Z} \times \mathcal{A} \rightarrow \mathbb{R}$, s.t. $\|g_{h}\|_{\infty} \leq B$ for all $h \in[H]$. We assume that MLE guarantee is hold. Then, for any policy $\pi$, we have
\$
\sum_{h=1}^H \mathbb{E}^{\widehat{\cP}}_{\pi}[g(z_h,a_h)]& \leq \sum_{h=1}^H \mathbb{E}^{\widehat{\cP}}_{z_{h-L-1},a_{h-L-1}\sim \pi}\biggl[\|\widehat{\phi}^{\top}(z_{h-L-1}, a_{h-L-1})\|_{\rho_{h-L-1}^{-1},\widehat{\phi}_{h-L-1}}\biggl]\\& \qquad\quad \cdot \sqrt{|A|^L k\cdot\mathbb{E}_{(\tilde{z}_{h},\tilde{a}_{h}) \sim \beta_{h}}\{[g(\tilde{z}_{h},\tilde{a}_{h})]^{2}\}+B^{2} \lambda_{k} d}+kB^2\zeta_k.
\$
\end{lemma}
\begin{proof}
For $h\in [H]$ and $h'\in [h-L+1,h]$, we define $\cX_{l}=\cS^{l}\times\cO^{l}\times\cA^{l-1}$ and
\$
x_{h'}=(s_{h-L+1: h'}, o_{h-L+1: h'}, a_{h-L+1: h'-1}),
\$
where $l=h'-h+L-1$.

Now we define the moment matching policy $\mu^{\pi, h}=\{\mu_{h'}^{\pi, h}: \mathcal{X}_l \rightarrow \Delta(\mathcal{A})\}_{h'=h-L+1}^h$. We set $\mu^{\pi,h}$ as following:
\$
\mu_{h'}^{\pi, h}(a_{h'} \mid x_{h'}):=\mathbb{E}^{\widehat{\cP}}_\pi[\pi_{h'}(a_{h'} \mid z_{h'}) \mid x_{h'}]~\text{for~}h'\leq h-1,~\text{and}~ \mu_h^{\pi,h}=\pi_h.
\$
Then we define policy $\tilde{\pi}_h$ which takes first $h-L$ actions from $\pi$ and remaining actions from $\mu^{\pi,h}$.

By Lemma B.2 in \citet{DBLP:journals/corr/abs-2202-03983}, we have $d^{\widehat{\cP},\pi}_h(z_h)=d^{\widehat{\cP},\tilde{\pi}^h}_h(z_h)$. Then we have $d^{\widehat{\cP},\pi}_h(z_h,a_h)=d^{\widehat{\cP},\tilde{\pi}^h}_h(z_h,a_h)$ since $\mu_h^{\pi,h}=\pi_h$. Hence we have $\mathbb{E}^{\widehat{\cP}}_{\pi}[g(z_h,a_h)]=\mathbb{E}^{\widehat{\cP}}_{\tilde{\pi}^h}[g(z_h,a_h)]$.

Since $\mu^{\pi,h}$ is independent of $s_{h-L+1}$, we have the $L$-step-back decomposition:
\$
&\mathbb{E}^{\widehat{\cP}}_{\tilde{\pi}^h}[g(z_h,a_h)]\\ &=\mathbb{E}^{\widehat{\cP}}_{z_{h-L},a_{h-L}\sim \pi}\biggl[\int_{s_{h-L+1}}\widehat{\phi}^{\top}(z_{h-L}, a_{h-L}) \widehat{\omega}(s_{h-L+1})(s_{h-L+1}) \cdot \mathbb{E}^{\widehat{\cP}}_{a_{h-L+1: h} \sim \mu^{\pi, h}}[g(z_h,a_h) \mid s_{h-L+1}]\operatorname{ds}_{h-L+1}\biggl]\\
&=\mathbb{E}^{\widehat{\cP}}_{z_{h-L},a_{h-L}\sim \pi}\widehat{\phi}^{\top}(z_{h-L}, a_{h-L})\cdot \int_{s_{h-L+1} \in \mathcal{S}} \widehat{\omega}(s_{h-L+1}) \cdot \mathbb{E}^{\widehat{\cP}}_{a_{h-L+1: h} \sim \mu^{\pi, h}}[g(z_h,a_h) \mid s_{h-L+1}]\operatorname{ds}_{h-L+1}\\&
\leq \mathbb{E}^{\widehat{\cP}}_{z_{h-L},a_{h-L}\sim \pi}\|\widehat{\phi}^{\top}(z_{h-L}, a_{h-L})\|_{\rho^{-1}_{h-L},\widehat{\phi}_{h-L}}\\&\qquad\quad \cdot \bigg\|\int_{s_{h-L+1} \in \mathcal{S}} \widehat{\omega}(s_{h-L+1})\cdot\mathbb{E}^{\widehat{\cP}}_{a_{h-L+1: h} \sim \mu^{\pi, h}}[g(z_h,a_h) \mid s_{h-L+1}]\operatorname{ds}_{h-L+1}\bigg\|_{\rho_{h-L},\widehat{\phi}_{h-L}},
\$
where the inequality is by Cauchy-Schwarz inequality.

Now we use $\tilde{g}(s_{h-L})$ to denote $\mathbb{E}^{\widehat{\cP}}_{a_{h-L+1: h} \sim \mu^{\pi, h}}[g(z_h,a_h) \mid s_{h-L}]$ for notational simplification. We have 
\$&\|\int_{s_{h-L+1} \in \mathcal{S}} \widehat{\omega}(s_{h-L+1})\tilde{g}(s_{h-L+1})\operatorname{ds}_{h-L+1}\|^2_{\rho_{h-L},\widehat{\phi}_{h-L}}\\
&=\biggl\{\int_{s_{h-L+1} \in \mathcal{S}} \widehat{\omega}(s_{h-L+1})\tilde{g}(s_{h-L+1})\operatorname{ds}_{h-L+1}\biggl\}^{\top}\\&\qquad\quad\biggl\{k \mathbb{E}_{(\tilde{z}_{h-L},\tilde{a}_{h-L}) \sim \rho_{h-L+1}}[\widehat{\phi}_{h-L+1}(\tilde{z}_{h-L},\tilde{a}_{h-L}) \widehat{\phi}^{\top}_{h-L+1}(\tilde{z}_{h-L},\tilde{a}_{h-L})]+\lambda_{k} I\biggl\}\\&\qquad\quad\cdot\biggl\{\int_{s_{h-L+1} \in \mathcal{S}} \widehat{\omega}(s_{h-L+1})\tilde{g}(s_{h-L+1})\operatorname{ds}_{h-L+1}\biggl\}\\
&\leq k \mathbb{E}_{(\tilde{z}_{h-L},\tilde{a}_{h-L}) \sim \rho_{h-L}}\biggl\{\bigg[\int_{s_{h-L+1} \in \mathcal{S}} \widehat{\omega}^{\top}(s_{h-L+1})\widehat{\phi}_{h-L+1}(\tilde{z}_{h-L},\tilde{a}_{h-L})\tilde{g}(s_{h-L+1})\operatorname{ds}_{h-L+1}\bigg]^{2}\biggl\}+B^{2} \lambda_{k} d\\
&=k \mathbb{E}_{(\tilde{z}_{h-L},\tilde{a}_{h-L}) \sim \rho_{h-L}}\biggl\{\bigg[\mathbb{E}_{s_{h-L+1}\sim \widehat{\cP}_{h-L+1}(\tilde{z}_{h-L},\tilde{a}_{h-L})}
\mathbb{E}_{a_{h-L+1: h} \sim \mu^{\pi, h}}[g(z_{h},a_{h}) \mid s_{h-L+1}]\bigg]^{2}\biggl\}+B^{2} \lambda_{k} d,
\$

where inequality is because $\|g_h\|\leq B$. Moreover, we have 
\$
&k \mathbb{E}_{(\tilde{z}_{h-L},\tilde{a}_{h-L}) \sim \rho_{h-L}}\biggl\{[\mathbb{E}_{s_{h-L+1}\sim \widehat{\cP}_{h-L+1}(\tilde{z}_{h-L},\tilde{a}_{h-L})}
\mathbb{E}_{a_{h-L+1: h} \sim \mu^{\pi, h}}[g(z_{h},a_{h}) \mid s_{h-L+1}]]^{2}\biggl\}\\
&\leq k \mathbb{E}_{(\tilde{z}_{h-L},\tilde{a}_{h-L}) \sim \rho_{h-L}}\biggl\{[\mathbb{E}_{s_{h-L+1}\sim \mathbb{P}^{\cP}_{h-L+1}(\cdot\mid\tilde{z}_{h-L},\tilde{a}_{h-L})}
\mathbb{E}_{a_{h-L+1: h} \sim \mu^{\pi, h}}[g(z_{h},a_{h}) \mid s_{h-L+1}]]^{2}\biggl\}+k B^2\zeta_k\\
&\leq  k \mathbb{E}_{(\tilde{z}_{h-L},\tilde{a}_{h-L}) \sim \rho_{h-L},s_{h-L+1}\sim \mathbb{P}^{\cP}_{h-L+1}(\cdot\mid\tilde{z}_{h-L},\tilde{a}_{h-L}),a_{h-L+1:h}\sim \mu^{\pi,h}}[g(z_{h},a_{h}) \mid s_{h-L+1}]^{2}+k B^2\zeta_k
\\&\leq |A|^L k \mathbb{E}_{(\tilde{z}_{h-L},\tilde{a}_{h-L}) \sim \rho_{h-L},s_{h-L+1}\sim \mathbb{P}^{\cP}_{h-L+1}(\cdot\mid\tilde{z}_{h-L},\tilde{a}_{h-L}),a_{h-L+1:h}\sim U(\cA)}[g(z_{h},a_{h}) ]^{2}+kB^2\zeta_k
\\
&=|A|^L k\cdot\mathbb{E}_{(\tilde{z}_{h},\tilde{a}_{h}) \sim \beta_{h}}\{[g(\tilde{z}_{h},\tilde{a}_{h})]^{2}\}+kB^2\zeta_k,
\$

where the first inequality is by MLE guarantee, the second inequality is by Jensen's inequality and the last  inequality is by importance sampling.

Then, the final statement is immediately concluded.
\end{proof}
\section{MLE guarantee}\label{mle}
In this section, we present the MLE guarantee used for $L$-step decodable POMDPs and $\gamma$-observable POMDPs. Regarding the proof, refer to \citet{agarwal2020flambe}. 
\begin{lemma}[MLE guarantee]
Set $\lambda_k=\Theta(d \log (|\cF|k/\delta))$, for any time step $h\in[H]$, denote $\rho_h$ as the joint distribution for $(z, a)$ in the dataset $\mathcal{D}$ of step $k$, with probability at least $1-\delta$ we have 
\begin{align}
&\mathbb{E}_{z, a \sim \rho}[\|\mathbb{P}_h^{\widehat{\cP}}(\cdot\mid z,a)^{\top}-\mathbb{P}_h^{\cP}(\cdot\mid z,a)^{\top}\|_1^2] \leq \zeta_k,
\end{align}
recall that $\zeta_k=O(\log(|\cF|k/\delta)/k)$.
\end{lemma}
\begin{lemma}[MLE guarantee for POMDP]\label{guarantee_POMDP}
Consider parameters defined in Theorem \ref{theorem:sample_POMDP} and time step $h$. Denote $\rho_h$ as the joint distribution for $(o_{3-2L:h},a_{3-2L:h},o'_{h+1})$ in the dataset $\mathcal{D}$ of size $k$. 
%Let $\widehat{\phi}=\operatorname{REPLEARN}(\mathcal{D}, \Phi_h, \mathcal{F}_h, \lambda_k, T_k, \ell_k)$. 
Then, with probability at least $1-\delta$ we have 
\$
&\mathbb{E}_{o_{3-2L:h},a_{3-2L:h} \sim \rho}[\|\mathbb{P}_h^{\widehat{\cP}}(\cdot\mid o_{3-2L:h},a_{3-2L:h})^{\top}-\mathbb{P}_h^{\cP}(\cdot\mid o_{3-2L:h},a_{3-2L:h})^{\top}\|_1^2] \leq \zeta_k=O\bigg(\frac{\log (k|\cF| / \delta)}{k}\bigg),
\$
in addition, we have
\$
&\mathbb{E}_{z_h,a_h \sim \rho}[\|\mathbb{P}_h^{\widehat{\cM}}(\cdot\mid z_h,a_h)^{\top}-\mathbb{P}_h^{\cM}(\cdot\mid z_h,a_h)^{\top}\|_1^2] \leq O\bigg(\frac{\log (k|\cF| / \delta)}{k}+\epsilon_1\bigg).
\$
\end{lemma}
\section{Missing Proofs of Section \ref{approximated_MDP}}\label{lem:value}
\subsection{Construction of the approximated MDP}\label{construct_q_function}
To construct this approximated MDP, we need to first calculate the belief state and approximated belief state, which is the conditional probability of state $s_h$ given the true transition and an action and observation sequence $\{o_{3-2L},a_{3-2L},\cdots,a_{h},o_{h}\}$ and $1\leq h\leq H$ respectively.

Consider a POMDP and a history $(o_{3-2L: h}, a_{3-2L: h-1}) \in \mathcal{H}_h$, the belief state $\mathbf{b}_h^{\mathcal{P}}(o_{3-2L: h}, a_{3-2L: h-1}) \in \Delta(\mathcal{S})$ is given by the distribution of the state $s_h$ conditioned on taking actions $a_{1:h-1}$ and observing $o_{1:h}$ in the first $h$ steps. Formally, the belief state is defined inductively as follows: $\mathbf{b}_1^{\mathcal{P}}(\varnothing)=b_1$, where $b_1$ is a properly chosen prior distribution whose precise form is deferred to appendix C. For $2 \leq h \leq H$ and any $(o_{3-2L: h}, a_{3-2L: h-1}) \in \mathcal{H}_h$, define
\$
&\mathbf{b}_h^{\mathcal{P}}(o_{3-2L: h}, a_{3-2L: h-1}):=U_{h-1}^{\mathcal{P}}(\mathbf{b}_{h-1}^{\mathcal{P}}(a_{1: h-2}, o_{2: h-1}) ; a_{h-1}, o_h),
\$
where for $\mathbf{b} \in \Delta(\mathcal{S}), a \in \mathcal{A}, o \in \mathcal{O},$ the belief update operator $ U_h^{\mathcal{P}}$ is defined as
\$
&U_h^{\mathcal{P}}(\mathbf{b} ; a, o)(s):=\frac{\mathbb{O}_{h+1}(o \mid s) \cdot \sum_{s^{\prime} \in \mathcal{S}} \mathbf{b}(s^{\prime}) \cdot \mathbb{P}_h(s \mid s^{\prime}, a)}{\sum_{x \in \mathcal{S}} \mathbb{O}_{h+1}(o \mid x) \sum_{s^{\prime} \in \mathcal{S}} \mathbf{b}(s^{\prime}) \cdot \mathbb{P}_h(x \mid s^{\prime}, a)},
\$
This operator calculates the belief state for $h+1$-step when the belief for the $h$-step is $b$, and after the agent takes action $a$ and receives the observation $o$.

Recall that  $\cP$ has a low-rank structure. For an extended POMDP, we have
\#\label{eq:pomdp_low_decom}
&\PP_h^{\cP}(o_{h+1} \mid o_{3-2L: h}, a_{3-2L: h})= \int_{\cS_{h+1}}\omega_h(s_{h+1})\cdot\mathbb{O}_{h+1}(o_{h+1}\mid s')\operatorname{ds}'\cdot\int_{\cS} \psi_h(s,a_h) b_h^{\cP}(o_{3-2L: h}, a_{3-2L: h-1})(s)\operatorname{ds}.
\#

For $h\in [H]$, $\tau_h\in \mathcal{H}_h$, $o_h\in \cO$ and $a_h\in \cZ_h$, we denote 
\#\label{eq:POMDP_realize}
& \int \psi_h(s_h,a_h) b_h^{\cP}(\tau_h)(s_h)\operatorname{ds}_h=\xi_h(\tau_h,a_h),
\# 
we define the approximated belief $\overline{b}_h(o_{h-L:h},a_{h-L:h-1})$ to approximate the true belief $b_h(o_{3-2L:h},a_{3-2L:h-1})$. 

For $\mathbf{b} \in \Delta(\mathcal{S})$ and $o\in \cO$, we define $B(\mathbf{b},o)$ as the operation that incorporates observation $o$ by 
\$
B(\mathbf{b},o)(s):=\frac{\mathbb{O}(o \mid s) \cdot \mathbf{b}(s)}{\sum_{x \in \mathcal{S}} \mathbb{O}(o \mid x) \sum_{s^{\prime} \in \mathcal{S}} \mathbf{b}(s')},
\$
which denotes the belief distribution after receiving the observation $o$ as the original belief distribution was $b$.

For an action and observation sequence $\{o_{3-2L},a_{3-2L},\cdots,o_{H},a_H\}$ and $2\leq h\leq H$. We define the approximated belief as :
\#
&\overline{b}_{h-L}=B(\tilde{b}^{h-L}_{0}, o_{h-L}),\label{eq:appro_belief}\\
&\overline{b}_{h-L+\tau}(o_{h-L: h-L+\tau}, a_{h-L: h-L-1+\tau})
=U^{\cP}_{h-L-1+\tau}(\overline{b}_{h-L-1+\tau}(o_{h-L: h-L-1+\tau}, a_{h-L: \tau}), o_{h-L+\tau}, a_{h-L-1+\tau}), 1 \leq \tau \leq l,\nonumber
\#
where the construction of the initial belief $\tilde{b}^{h-L}_0$ can be found in Appendix \ref{construct}.

We have the following lemma to give an upper bound for the difference between the approximate belief and the true belief. The detailed proof can be found in Appendix \ref{proof:app-exist}.

\begin{lemma}\label{masa}
For any policy $\pi\in \Delta(\Acal)$ and $h\geq 2-L$, the approximate belief $\overline{b}$ defined in \eqref{eq:appro_belief} satisfy that
\$
&\mathbb{E}^{\pi,\cP}[\|b_{h+L}(o_{3-2L: h+L}, a_{3-2L: h+L-1})-\overline{b}_{h+L}(z_h,a_h)\|_{1}]\leq \epsilon_1,~\text{where}~(z_h,a_h)=(o_{h: h+L}, a_{h: h+L-1}).
\$
\end{lemma}
We have
\#\label{eq:mugamma}\int\omega_h (s'_{h})\cdot\mathbb{O}_h(o_{h}\mid s'_{h})\operatorname{ds}'_{h}=\mu_h(o_{h})\#

Now we can construct $\cM = (\Scal,\Acal, \Pcal, r)$ as following
\$
&\PP_h^{\cM}(o_{h+1} \mid o_{h+1-L: h}, a_{h+1-L: h})\\& = \int_{\cS_{h+1}}\omega_{h+1}(s_{h+1})\cdot\mathbb{O}_{h+1}(o_{h+1}\mid s_{h+1})\operatorname{ds}_{h+1}\cdot\int_{\cS_h}\psi(s_h,a_h) \overline{b}_h^{\cP}(o_{h-L+1: h}, a_{h-L+1: h-1})(s_h)\operatorname{ds}_h.
\$
We denote
\#\label{eq:phi}
\int \psi(s_h,a_h) \overline{b}_h^{\cP}(z_h)(s_h)\operatorname{ds}_h=\phi(z_h,a_h).\#
We define $(\mu,\phi,\xi)=q(\mathbb{O},\omega,\psi)$ by \eqref{eq:POMDP_realize}, \eqref{eq:mugamma} and \eqref{eq:phi}.
\subsection{Construction of $\tilde{b}^h_0$}\label{construct}
Now we present the construction of $\tilde{b}^h_0$ for $h\in[H]$, which is done by \citet{uehara2022provably}.
\begin{lemma}[G-optimal design \citep{uehara2022provably}]\label{gproperty}
 Suppose $\mathcal{X} \in \mathbb{R}^d$ is a compact set. There exists a distribution $\rho$ over $\mathcal{X}$ such that:
(i) $\rho$ is supported on at most $d(d+1) / 2$ points.
(ii) For any $x^{\prime} \in \mathcal{X}$, we have $x^{\prime \top} \mathbb{E}_{x \sim \rho}\left[x x^{\top}\right]^{-1} x^{\prime} \leq d$.
\end{lemma}
By Lemma \ref{gproperty}, there exist $\rho_h\in \Delta(\cS_h\times\cA_h)$ to be the G-optimal design for $\psi_h(s,a)$. Denote the support of $\rho_h$ to be $S_{\rho,h}$, $S_{\rho,h}$ has most $d(d+1)/2$ points. We denote $S_{\rho,h}=\{s^i_h,a^i_h\}^{|S_{\rho,h}|}_{i=1}$.

We set $\tilde{b}^h_0$ as
\$
\tilde{b}^h_0(\cdot):=\sum_{\tilde{s}, \tilde{a}} \rho_h(\tilde{s}, \tilde{a}) \mathbb{P}^{\cP}_h(\cdot \mid \tilde{s}, \tilde{a})=\sum_{i=1}^{\left|S_{\rho,h}\right|} \rho\left(s^i_h, a^i_h\right) \mathbb{P}^{\cP}_h\left(\cdot \mid s^i_h, a^i_h\right)
\$
\subsection{Proof for Proposition \ref{prop:app-exist}}\label{proof:app-exist}

First, we prove the existence of the approximated MDP $\cM$ in Proposition \ref{prop:app-exist}.

\begin{proof}[Proof of Lemma \ref{masa}]
The proof is same to Theorem 14 of \citet{https://doi.org/10.48550/arxiv.2206.12020}, the only difference is that the process is start from $3-2L$ instead of $1$.
\end{proof}
Now we construct the approximated MDP $\cM$. For a state $z_{h}=(o_{h-L: h}, a_{h-L: h-1}) \in \mathcal{Z}$, action $a_{h}$, and subsequent observation $o_{h+1} \in \mathcal{O}$, define
$$
\mathbb{P}_{h}^{\mathcal{M}}((o_{h-L+1: h+1}, a_{h-L+1: h}) \mid z_{h}, a_{h}):=e_{o_{h+1}}^{\top} \cdot \mathbb{O}_{h+1}^{\mathcal{P}} \cdot \mathbb{T}_{h}^{\mathcal{P}}(a_{h}) \cdot \overline{b}_{h}(o_{h-L:h},a_{h-L:h-1}) .
$$
Hence we can define 
\$
\phi(z_h,a_h)=\mathbb{T}_{h}^{\mathcal{P}}(a_{h}) \cdot \overline{b}_{h},~\mu(o_{h+1})=e_{o_{h+1}}^{\top} \cdot \mathbb{O}_{h+1}^{\mathcal{P}}
\$

Note that for the POMDP $\cP$, we have 
\$
\mathbb{P}_{o_{h+1} \sim \pi}^{\mathcal{P}}(o_{h+1} \mid o_{1: h}, a_{1: h})=e_{o_{h+1}}^{\top} \cdot \mathbb{O}_{h+1}^{\mathcal{P}} \cdot \mathbb{T}_{h}^{\mathcal{P}}(a_{h}) \cdot \mathbf{b}_{h}^{\mathcal{P}}(o_{1: h-1}, a_{1: h}).
\$
Lemma \ref{masa} shows that $\mathbb{E}[\|b_{h+L}(o_{3-2L: h+L}, a_{3-2L: h+L-1})-\overline{b}_{h+L}(o_{h: h+L}, a_{h: h+L-1})\|_{1} ; a_{1: h+L-1} \sim \pi]$ is small in expectation under any policy $\pi$, so we have 
\$
\mathbb{E}_{a_{1: h}, o_{2: h} \sim \pi}\|\mathbb{P}_{h}^{\mathcal{M}}(o_{h+1} \mid z_{h}, a_{h})-\mathbb{P}_{h}^{\mathcal{P}}(o_{h+1} \mid o_{1: h}, a_{1: h})\|_1 \leq \epsilon_1.
\$
With the approximated MDP and approximated belief, we have the following lemma.
\begin{lemma}\label{transfer from pomdp to mdp}
For any function $g:\cS \rightarrow \mathbb{R}$, $\|g\|_{\infty}\leq 1$, we have
\$
\mathbb{E}_{a_{1: h}, o_{2: h} \sim \pi}\int_{s_{h+1}}\mathbb{P}_{h}^{\mathcal{P}}(s_{h+1} \mid o_{1: h}, a_{1: h})g(s_{h+1})\operatorname{ds} _{h+1}&\leq \mathbb{E}_{a_{1: h}, o_{2: h} \sim \pi}\int_{s_{h+1}}\omega(s_{h+1})^{\top}\phi(z_h,a_h)g(s_{h+1})\operatorname{ds}_{h+1}+\epsilon_1.
\$
\end{lemma}
\begin{proof}
We have 
\$
&\mathbb{E}_{a_{1: h}, o_{2: h} \sim \pi}\int_{s_{h+1}}\mathbb{P}_{h}^{\mathcal{P}}(s_{h+1} \mid o_{1: h}, a_{1: h})g(s_{h+1})\operatorname{ds}_{h+1}\\&\quad =\mathbb{E}_{a_{1: h}, o_{2: h} \sim \pi}\int_{s_{h+1}}\omega(s_{h+1}) ^{\top} \int_{s_h} \psi(s_h,a_{h}) b_h^{\cP}(o_{3-2L: h}, a_{3-2L: h})(s_h)g(s_{h+1})\operatorname{ds}_h\operatorname{ds} _{h+1}\\&\quad\leq  \mathbb{E}_{a_{1: h}, o_{2: h} \sim \pi}\int_{s_{h+1}}\omega(s_{h+1}) ^{\top} \int_{s_h} \psi(s_h,a_{h})| (b_h^{\cP}(o_{3-2L: h}, a_{3-2L: h})-\overline{b}_{h}(z_h,a_h))(s_h)|)|g(s_{h+1})|\operatorname{ds}_h\operatorname{ds} _{h+1}\\&\qquad\quad+\mathbb{E}_{a_{1: h}, o_{2: h} \sim \pi}\int_{s_{h+1}}\omega(s_{h+1}) ^{\top} \int_{s_h} \psi(s_h,a_{h}) \overline{b}_h(z_h,a_h)(s_h)g(s_{h+1})\operatorname{ds}_h\operatorname{ds}_{h+1},
\$
where the inequality is by Cauchy-Schwarz inequality.

Since we have $\int_{s_{h+1}}\omega(s_{h+1}) ^{\top} \psi(s_h,a_{h})|g(s_{h+1})|\operatorname{ds} _{h+1}\leq 1$ for any $s_h$, so we have
\$
&\mathbb{E}_{a_{1: h}, o_{2: h} \sim \pi}\int_{s_{h+1}}\omega(s_{h+1}) ^{\top} \int_{s_h} \psi(s_h,a_{h})| (b_h^{\cP}(o_{3-2L: h}, a_{3-2L: h})-\overline{b}_{h}(z_h,a_h))(s_h)||g(s_{h+1})|\operatorname{ds}_h\operatorname{ds}_{h+1}\\&\quad\leq \mathbb{E}_{a_{1: h}, o_{2: h} \sim \pi}\int_{s_h}\big| (b_h^{\cP}(o_{3-2L: h}, a_{3-2L: h})-\overline{b}_{h}(z_h,a_h))(s_h)\big|\operatorname{ds}_h\\&\quad\leq \epsilon_1,
\$
where the last inequality is by Lemma \ref{masa}.

Remember that $\int_{\cS} \psi(s,a_{h-L})\overline{b}_h(z_{h-L},a_{h-L})(s)\operatorname{ds}=\phi_h(z_h,a_h)$, so we have 
\$
&\mathbb{E}_{a_{1: h}, o_{2: h} \sim \pi}\int_{s_{h+1}}\mathbb{P}_{h}^{\mathcal{P}}(s_{h+1} \mid o_{1: h}, a_{1: h})g(s_{h+1})\operatorname{ds} _{h+1}\\&\quad\leq \mathbb{E}_{a_{1: h}, o_{2: h} \sim \pi}\int_{s_{h+1}}\omega(s_{h+1})^{\top}\phi(z_h,a_h)g(s_{h+1})\operatorname{ds}_{h+1}+\epsilon_1,
\$
which concludes the proof.
\end{proof}
\subsection{Proof of Lemma \ref{dif}}
\begin{proof}[Proof of Lemma \ref{dif}]
$$
\begin{aligned}
|V_{1}^{\pi, \mathcal{P}, r}(\varnothing)-V_{1}^{\pi, \mathcal{M}, r}(\varnothing)|
&=\bigg|\mathbb{E}_{(a_{1: H-1}, o_{2: H}) \sim \pi}^{\mathcal{P}}\bigg[\sum_{h=1}^{H}((\mathbb{P}_{h}^{\mathcal{P}}-\mathbb{P}_{h}^{\mathcal{M}})(V_{h+1}^{\pi, \mathcal{M}, r}+r_{h+1}))(a_{1: h}, o_{2: h})\bigg]\bigg|\\ &=\bigg|\mathbb{E}^{\mathcal{P}}_{(a_{1: H-1}, o_{2: H}) \sim \pi}\bigg[\sum_{h=1}^{H-1}((\mathbb{P}_{h}^{\mathcal{P}}-\mathbb{P}_{h}^{\mathcal{M}})(V_{h+1}^{\pi, \mathcal{M}, r}+r_{h+1}))(a_{1: h}, o_{2: h})\bigg]\bigg| \\&\leq \frac{H}{2} \cdot \mathbb{E}_{(o_{3-2L: h}, a_{3-2L: h-1}) \sim \pi}^{\mathcal{P}}\bigg[\sum_{h=1}^{H-1}\|\mathbb{P}_{h}^{\mathcal{P}}(\cdot \mid a_{1: h}, o_{2: h})-\mathbb{P}_{h}^{\mathcal{M}}(\cdot \mid z_{h}, a_{h})\|_{1}\bigg]\\&\leq \frac{H^2\epsilon_1}{2},\end{aligned}
$$
where the first inequality uses the fact that $|(V_{h+1}^{\pi, \mathcal{M}, r}+r_{h+1})(a_{1: h}, o_{2: h+1})| \leq H \text { for all } a_{1: h}, o_{2: h+1} \text {, }$ the second inequality is by Proposition \ref{prop:app-exist}.
\end{proof}
\begin{lemma}[$L$-step back inequality for the true POMDP]\label{$L$-step no}
Consider a set of functions $\{g_{h}\}_{h=0}^{H}$ that satisfies $g_{h} \in \mathcal{Z} \times \mathcal{A} \rightarrow \mathbb{R}$, s.t. $\|g_{h}\|_{\infty} \leq B$ for all $h \in[H]$. Then, for any policy $\pi$, we have
\$
\sum_{h=1}^H \mathbb{E}^{\cP}_{\pi}[g(z_h,a_h)]& \leq \sum_{h=1}^H \mathbb{E}^{\cP}_{z_{h-L-1},a_{h-L-1}\sim \pi}\|\phi^{\top}(z_{h-L-1}, a_{h-L-1})\|_{\beta^{-1}_{h-L-1},\phi_{h-L-1}}\\& \qquad\quad \cdot \sqrt{|A|^L k\cdot\mathbb{E}^{\cP}_{(\tilde{z}_{h},\tilde{a}_{h}) \sim \gamma_{h}}\{[g(\tilde{z}_{h},\tilde{a}_{h})]^{2}\}+B^{2} \lambda_{k} d+kB^2\epsilon_1}+B\epsilon_1.
\$
\end{lemma}
\begin{proof}
For $h\in [H]$ and $h'\in [h-L+1,h]$, we define $\cX_{l}=\cS^{l}\times\cO^{l}\times\cA^{l-1}$ and
\$
x_{h'}=(s_{h-L+1: h'}, o_{h-L+1: h'}, a_{h-L+1: h'-1}),
\$
where $l=h'-h+L-1$.

Now we define the moment matching policy $\mu^{\pi, h}=\{\mu_{h'}^{\pi, h}: \mathcal{X}_l \rightarrow \Delta(\mathcal{A})\}_{h'=h-L+1}^h$. We set $\mu^{\pi,h}$ as following:
\$
\mu_{h'}^{\pi, h}(a_{h'} \mid x_{h'}):=\mathbb{E}^{\cP}_\pi[\pi_{h'}(a_{h'} \mid z_{h'}) \mid x_{h'}]~\text{for~}h'\leq h-1,~\text{and}~ \mu_h^{\pi,h}=\pi_h.
\$
Then we define policy $\tilde{\pi}^h$ which takes first $h-L$ actions from $\pi$ and remaining actions from $\mu^{\pi,h}$.

By Lemma B.2 in \citet{DBLP:journals/corr/abs-2202-03983}, we have $d^{\cP,\pi}_h(z_h)=d^{\cP,\tilde{\pi}^h}_h(z_h)$. Then we have $d^{\cP,\pi}_h(z_h,a_h)=d^{\cP,\tilde{\pi}^h}_h(z_h,a_h)$ since $\mu_h^{\pi,h}=\pi_h$. Hence we have $\mathbb{E}^{\cP}_{\pi}[g(z_h,a_h)]=\mathbb{E}^{\cP}_{\tilde{\pi}^h}[g(z_h,a_h)]$.

Since $\mu^{\pi,h}$ is independent of $s_{h-L+1}$, we have the $L$-step-back decomposition:
\$
&\mathbb{E}^{\cP}_{\tilde{\pi}^h}[g(z_{h},a_{h})]\\&=\mathbb{E}^{\cP}_{o_{3-2L:h-L},a_{3-2L,h-L}\sim \pi}\biggl[\int_{s_{h-L+1}}\mathbb{P}(s_{h-L+1}\mid o_{3-2L:h-L},a_{3-2L,h-L})\cdot\mathbb{E}^{\cP}_{a_{h-L+1: h} \sim \mu^{\pi, h}}[g(z_h,a_h) \mid s_{h-L+1}]\operatorname{ds}_{h-L+1}\biggl]\\&\leq \mathbb{E}^{\cP}_{z_{h-L},a_{h-L}\sim \pi}\biggl[\int_{s_{h-L+1}}\phi^{\top}(z_{h-L}, a_{h-L}) \omega(s_{h-L+1}) \cdot \mathbb{E}^{\cP}_{a_{h-L+1: h} \sim \mu^{\pi, h}}[g(z_h,a_h) \mid s_{h-L+1}]\operatorname{ds}_{h-L+1}\biggl]+B\epsilon_1\\
&=\mathbb{E}^{\cP}_{z_{h-L},a_{h-L}\sim \pi}\phi^{\top}(z_{h-L}, a_{h-L})\cdot \int_{s_{h-L+1} \in \mathcal{S}} \omega(s_{h-L+1}) \cdot \mathbb{E}^{\cP}_{a_{h-L+1: h} \sim \mu^{\pi, h}}[g(z_h,a_h) \mid s_{h-L+1}]\operatorname{ds}_{h-L+1}+B\epsilon_1\\&
\leq \mathbb{E}^{\cP}_{z_{h-L},a_{h-L}\sim \pi}\|\phi^{\top}(z_{h-L}, a_{h-L})\|_{\beta^{-1}_{h-L},\phi_{h-L}}\\&\qquad\quad \cdot \bigg\|\int_{s_{h-L+1} \in \mathcal{S}} \omega(s_{h-L+1})\cdot\mathbb{E}^{\cP}_{a_{h-L+1: h} \sim \mu^{\pi, h}}[g(z_h,a_h) \mid s_{h-L+1}]\operatorname{ds}_{h-L+1}\bigg\|_{\beta_{h-L},\phi_{h-L}}+B\epsilon_1,
\$
where the first inequality is because Lemma \ref{transfer from pomdp to mdp}.

Now we use $\tilde{g}(s_{h-L})$ to denote $\mathbb{E}^{\cP}_{a_{h-L+1: h} \sim \mu^{\pi, h}}[g(z_h,a_h) \mid s_{h-L}]$ for notational simplification. We have 

\$&\bigg\|\int_{s_{h-L+1} \in \mathcal{S}} \omega(s_{h-L+1})\tilde{g}(s_{h-L+1})\operatorname{ds}_{h-L+1}\bigg\|^2_{\beta_{h-L},\phi_{h-L}}\\
&=\biggl\{\int_{s_{h-L+1} \in \mathcal{S}} \omega(s_{h-L+1})\tilde{g}(s_{h-L+1})\operatorname{ds}_{h-L+1}\biggl\}^{\top}\\&\qquad\quad\cdot\biggl\{k \mathbb{E}_{(\tilde{z}_{h-L},\tilde{a}_{h-L}) \sim \beta_{h-L+1}}[\phi_{h-L}(\tilde{z}_{h-L},\tilde{a}_{h-L}) \phi^{\top}_{h-L}(\tilde{z}_{h-L},\tilde{a}_{h-L})]+\lambda_{k} I\biggl\}\\&\qquad\quad\cdot\biggl\{\int_{s_{h-L+1} \in \mathcal{S}} \omega(s_{h-L+1})\tilde{g}(s_{h-L+1})\operatorname{ds}_{h-L+1}\biggl\}\\
&\leq k \mathbb{E}_{(\tilde{z}_{h-L},\tilde{a}_{h-L}) \sim \beta_{h-L}}\biggl\{\bigg[\int_{s_{h-L+1} \in \mathcal{S}} \omega^{\top}(s_{h-L+1})\phi_{h-L+1}(\tilde{z}_{h-L},\tilde{a}_{h-L})\tilde{g}(s_{h-L+1})\operatorname{ds}_{h-L+1}\bigg]^{2}\biggl\}+B^{2} \lambda_{k} d\\
&=k \mathbb{E}_{(\tilde{o}_{3-2L:h-L},\tilde{a}_{3-2L:h-L}) \sim \beta_{h-L}}\biggl\{\bigg[\int_{s_{h-L+1} \in \mathcal{S}} \omega^{\top}(s_{h-L+1})\phi_{h-L+1}(\tilde{z}_{h-L},\tilde{a}_{h-L})\tilde{g}(s_{h-L+1})\operatorname{ds}_{h-L+1}\bigg]^{2}\biggl\}+B^{2} \lambda_{k} d,
\$

where the inequality is because $\|g_h\|\leq B$.

Moreover, we have
\$
&k \mathbb{E}_{(\tilde{o}_{3-2L:h-L},\tilde{a}_{3-2L:h-L}) \sim \beta_{h-L}}\biggl\{[\int_{s_{h-L+1} \in \mathcal{S}} \omega^{\top}(s_{h-L+1})\phi_{h-L+1}(\tilde{z}_{h-L},\tilde{a}_{h-L})\tilde{g}(s_{h-L+1})\operatorname{ds}_{h-L+1}]^{2}\biggl\}\\
&\leq k \mathbb{E}_{(\tilde{o}_{3-2L:h-L},\tilde{a}_{3-2L:h-L}) \sim \beta_{h-L}}\biggl\{[\mathbb{E}_{s_{h-L+1}\sim \mathbb{P}^{\cP}_{h-L+1}(\tilde{o}_{3-2L:h-L},\tilde{a}_{3-2L:h-L})}
\mathbb{E}^{\cP}_{a_{h: h} \sim \mu^{\pi, h}}[g(z_{h},a_{h}) \mid s_{h-L+1}]]^{2}\biggl\}
+kB^2\epsilon_1\\&\leq  k \mathbb{E}_{(\tilde{o}_{3-2L:h-L},\tilde{a}_{3-2L:h-L}) \sim \beta_{h-L},s_{h-L+1}\sim \mathbb{P}^{\cP}_{h-L+1}(\tilde{o}_{3-2L:h-L},\tilde{a}_{3-2L:h-L}),a_{h-L+1:h}\sim \mu^{\pi,h}}[g(z_{h},a_{h}) \mid s_{h-L+1}]^{2}+kB^2\epsilon_1
\\
&\leq |A|^L k \mathbb{E}_{(\tilde{o}_{3-2L:h-L},\tilde{a}_{3-2L:h-L}) \sim \beta_{h-L},a_{h-L+1:h}\sim U(\cA)}[g(z_{h},a_{h}) ]^{2}+kB^2\epsilon_1\
\\
&=|A|^L k\cdot\mathbb{E}_{(\tilde{z}_{h},\tilde{a}_{h}) \sim \gamma_{h}}\{[g(\tilde{z}_{h},\tilde{a}_{h})]^{2}\}+kB^2\epsilon_1,
\$

where the first inequality is by Lemma \ref{transfer from pomdp to mdp}, the second inequality is by Jensen's inequality and the last inequality is by importance sampling, the equation is by the definition of $\gamma_h$.

Then, the final statement is immediately concluded.
\end{proof}
\begin{lemma}[$L$-step back inequality for the learned POMDP]\label{$L$-step learn no}
Consider a set of functions $\{g_{h}\}_{h=0}^{H}$ that satisfies $g_{h} \in \mathcal{Z} \times \mathcal{A} \rightarrow \mathbb{R}$, s.t. $\|g_{h}\|_{\infty} \leq B$ for all $h \in[H]$. Then, for any policy $\pi$, we have
\$
\sum_{h=1}^H \mathbb{E}^{\widehat{\cP}}_{\pi}[g(z_h,a_h)]& \leq \sum_{h=1}^H \mathbb{E}^{\widehat{\cP}}_{z_{h-L-1},a_{h-L-1}\sim \pi}[\|\phi^{\top}(z_{h-L-1}, a_{h-L-1})\|_{\rho_{h-L-1}^{-1},\phi_{h-L-1}}]\\& \qquad\quad \cdot \sqrt{|A|^L k\cdot\mathbb{E}^{\cP}_{(\tilde{z}_{h},\tilde{a}_{h}) \sim \beta_{h}}\{[g(\tilde{z}_{h},\tilde{a}_{h})]^{2}\}+B^{2} \lambda_{k} d+kB^2\epsilon_1}+B\epsilon_1.
\$
\end{lemma}
\begin{proof}
For $h\in [H]$ and $h'\in [h-L+1,h]$, we define $\cX_{l}=\cS^{l}\times\cO^{l}\times\cA^{l-1}$ and
\$
x_{h'}=(s_{h-L+1: h'}, o_{h-L+1: h'}, a_{h-L+1: h'-1}),
\$
where $l=h'-h+L-1$.

Now we define the moment matching policy $\mu^{\pi, h}=\{\mu_{h'}^{\pi, h}: \mathcal{X}_l \rightarrow \Delta(\mathcal{A})\}_{h'=h-L+1}^h$. We set $\mu^{\pi,h}$ as following:
\$
\mu_{h'}^{\pi, h}(a_{h'} \mid x_{h'}):=\mathbb{E}^{\widehat{\cP}}_\pi[\pi_{h'}(a_{h'} \mid z_{h'}) \mid x_{h'}]~\text{for~}h'\leq h-1,~\text{and}~ \mu_h^{\pi,h}=\pi_h.
\$
Then we define policy $\tilde{\pi}_h$ which takes first $h-L$ actions from $\pi$ and remaining actions from $\mu^{\pi,h}$.

By Lemma B.2 in \citet{DBLP:journals/corr/abs-2202-03983}, we have $d^{\widehat{\cP},\pi}_h(z_h)=d^{\widehat{\cP},\tilde{\pi}^h}_h(z_h)$. Then we have $d^{\widehat{\cP},\pi}_h(z_h,a_h)=d^{\widehat{\cP},\tilde{\pi}^h}_h(z_h,a_h)$ since $\mu_h^{\pi,h}=\pi_h$. Hence we have $\mathbb{E}^{\widehat{\cP}}_{\pi}[g(z_h,a_h)]=\mathbb{E}^{\widehat{\cP}}_{\tilde{\pi}^h}[g(z_h,a_h)]$.

Since $\mu^{\pi,h}$ is independent of $s_{h-L+1}$, we have the $L$-step-back decomposition:
\$
&\mathbb{E}^{\widehat{\cP}}_{\tilde{\pi}^h}[g(z_{h},a_{h})]\\&=\mathbb{E}^{\widehat{\cP}}_{o_{3-2L:h-L},a_{3-2L,h-L}\sim \pi}\biggl[\int_{s_{h-L+1}}\mathbb{P}(s_{h-L+1}\mid o_{3-2L:h-L},a_{3-2L,h-L})\cdot\mathbb{E}^{\widehat{\cP}}_{a_{h-L+1: h} \sim \mu^{\pi, h}}[g(z_h,a_h) \mid s_{h-L+1}]\operatorname{ds}_{h-L+1}\biggl]\\&\leq \mathbb{E}^{\widehat{\cP}}_{z_{h-L},a_{h-L}\sim \pi}\biggl[\int_{s_{h-L+1}}\phi^{\top}(z_{h-L}, a_{h-L}) \omega(s_{h-L+1}) \cdot \mathbb{E}^{\widehat{\cP}}_{a_{h-L+1: h} \sim \mu^{\pi, h}}[g(z_h,a_h) \mid s_{h-L+1}]\operatorname{ds}_{h-L+1}\biggl]+B\epsilon_1\\
&=\mathbb{E}^{\widehat{\cP}}_{z_{h-L},a_{h-L}\sim \pi}\phi^{\top}(z_{h-L}, a_{h-L})\cdot \int_{s_{h-L+1} \in \mathcal{S}} \omega(s_{h-L+1}) \cdot \mathbb{E}^{\widehat{\cP}}_{a_{h-L+1: h} \sim \mu^{\pi, h}}[g(z_h,a_h) \mid s_{h-L+1}]\operatorname{ds}_{h-L+1}+B\epsilon_1\\&
\leq \mathbb{E}^{\widehat{\cP}}_{z_{h-L},a_{h-L}\sim \pi}\|\phi^{\top}(z_{h-L}, a_{h-L})\|_{\rho^{-1}_{h-L},\phi_{h-L}}\\&\qquad\quad \cdot \|\int_{s_{h-L+1} \in \mathcal{S}} \omega(s_{h-L+1})\cdot\mathbb{E}^{\widehat{\cP}}_{a_{h-L+1: h} \sim \mu^{\pi, h}}[g(z_h,a_h) \mid s_{h-L+1}]\operatorname{ds}_{h-L+1}\|_{\rho_{h-L},\phi_{h-L}}+B\epsilon_1,
\$
where the first inequality is because Lemma \ref{transfer from pomdp to mdp}.

Now we use $\tilde{g}(s_{h-L})$ to denote $\mathbb{E}^{\widehat{\cP}}_{a_{h-L+1: h} \sim \mu^{\pi, h}}[g(z_h,a_h) \mid s_{h-L}]$ for notational simplification. We have 

\$&\|\int_{s_{h-L+1} \in \mathcal{S}} \omega(s_{h-L+1})\tilde{g}(s_{h-L+1})\operatorname{ds}_h\|^2_{\rho_{h-L},\phi_{h-L}}\\
&=\biggl\{\int_{s_{h-L+1} \in \mathcal{S}} \omega(s_{h-L+1})\tilde{g}(s_{h-L+1})\operatorname{ds}_h\biggl\}^{\top}\\&\qquad\quad\cdot\biggl\{k \mathbb{E}_{(\tilde{z}_{h-L},\tilde{a}_{h-L}) \sim \rho_{h-L+1}}[\phi_{h-L}(\tilde{z}_{h-L},\tilde{a}_{h-L}) \phi^{\top}_{h-L}(\tilde{z}_{h-L},\tilde{a}_{h-L})]+\lambda_{k} I\biggl\}\\&\qquad\quad\cdot\biggl\{\int_{s_{h-L+1} \in \mathcal{S}} \omega(s_{h-L+1})\tilde{g}(s_{h-L+1})\operatorname{ds}_h\biggl\}\\
&\leq k \mathbb{E}_{(\tilde{z}_{h-L},\tilde{a}_{h-L}) \sim \rho_{h-L}}\biggl\{[\int_{s_{h-L+1} \in \mathcal{S}} \omega^{\top}(s_{h-L+1})\phi_{h-L+1}(\tilde{z}_{h-L},\tilde{a}_{h-L})\tilde{g}(s_{h-L+1})\operatorname{ds}_h]^{2}\biggl\}+B^{2} \lambda_{k} d\\
&=k \mathbb{E}_{(\tilde{o}_{3-2L:h-L},\tilde{a}_{3-2L:h-L}) \sim \rho_{h-L}}\biggl\{[\int_{s_{h-L+1} \in \mathcal{S}} \omega^{\top}(s_{h-L+1})\phi_{h-L+1}(\tilde{z}_{h-L},\tilde{a}_{h-L})\tilde{g}(s_{h-L+1})\operatorname{ds}_h]^{2}\biggl\}+B^{2} \lambda_{k} d\\&\leq k \mathbb{E}_{(\tilde{o}_{3-2L:h-L},\tilde{a}_{3-2L:h-L}) \sim \rho_{h-L}}\biggl\{[\mathbb{E}_{s_h\sim \widehat{\cP}_h(\tilde{o}_{3-2L:h-L},\tilde{a}_{3-2L:h-L})}
\mathbb{E}^{\widehat{\cP}}_{a_{h: h} \sim \mu^{\pi, h}}[g(z_{h},a_{h}) \mid s_{h-L+1}]]^{2}\biggl\}+B^{2} \lambda_{k} d
+kB^2\epsilon_1,
\$

where the first inequality is because $\|g_h\|\leq B$, the second inequality is by Lemma \ref{transfer from pomdp to mdp}. Moreover, we have
\$&k \mathbb{E}_{(\tilde{o}_{3-2L:h-L},\tilde{a}_{3-2L:h-L}) \sim \rho_{h-L}}\biggl\{[\mathbb{E}_{s_h\sim \widehat{\cP}_h(\tilde{o}_{3-2L:h-L},\tilde{a}_{3-2L:h-L})}
\mathbb{E}^{\widehat{\cP}}_{a_{h: h} \sim \mu^{\pi, h}}[g(z_{h},a_{h}) \mid s_{h-L+1}]]^{2}\biggl\}\\
&\leq k \mathbb{E}_{(\tilde{o}_{3-2L:h-L},\tilde{a}_{3-2L:h-L}) \sim \rho_{h-L}}\biggl\{[\mathbb{E}_{s_h\sim \mathbb{P}^{\cP}_h(\tilde{o}_{3-2L:h-L},\tilde{a}_{3-2L:h-L})}
\mathbb{E}_{a_{h: h} \sim \mu^{\pi, h}}[g(z_{h},a_{h}) \mid s_{h-L+1}]]^{2}\biggl\}
\\&\leq  k \mathbb{E}_{(\tilde{o}_{3-2L:h-L},\tilde{a}_{3-2L:h-L}) \sim \rho_{h-L},s_h\sim \mathbb{P}^{\cP}_h(\tilde{o}_{3-2L:h-L},\tilde{a}_{3-2L:h-L}),a_{h-L+1:h}\sim \mu^{\pi,h}}[g(z_{h},a_{h}) \mid s_{h-L+1}]^{2}
\\
&\leq |A|^L k \mathbb{E}_{(\tilde{o}_{3-2L:h-L},\tilde{a}_{3-2L:h-L}) \sim \beta_{h-L},a_{h-L+1:h}\sim U(\cA)}[g(z_{h},a_{h}) ]^{2}
\\
&=|A|^L k\cdot\mathbb{E}_{(\tilde{z}_{h},\tilde{a}_{h}) \sim \beta_{h}}\{[g(\tilde{z}_{h},\tilde{a}_{h})]^{2}\},
\$

where the first inequality is by Lemma \ref{guarantee}, the second inequality is by Jensen's inequality, the last inequality is by importance sampling and the equation is by the definition of $\beta_h$.

Then, the final statement is immediately concluded.
\end{proof}

\begin{lemma}\label{transfer}
For any $\pi$, $h$, we have 
\$
\|d^{\pi}_{\cP,h}-d^{\pi}_{\cM,h}\|_{TV}\leq h\epsilon_1.
\$
\end{lemma}
\begin{proof}
The proof is similar to Lemma D.1 in \citet{https://doi.org/10.48550/arxiv.2205.14571}. The only difference is that our condition is Proposition \ref{prop:app-exist}.

\end{proof}
Next, we prove the almost optimism lemma restated below.
\begin{lemma}[Almost Optimism for $\gamma$-observable  POMDPs]\label{lem:almost_pomdp}
Consider an episode $k~(1 \leq k\leq K)$ and set
$$
\alpha_{k}=\sqrt{k|\mathcal{A}|^{L} \zeta_{k}+4 \lambda_{k} d+k\epsilon_1} / c, \quad \lambda_{k}=O(d \log (|\cF| k / \delta)) .
$$
where $c$ is an absolute constant. Conditioning on the event defined in Lemma \ref{guarantee}, with probability $1-\delta$,
$$
V^{\pi^{*},\widehat{\cM}_k,r+\widehat{b}^k}-V^{\pi^{*},\cM,r} \geq-\frac{c\alpha_kL}{\sqrt{k}}-\cO(H^2\epsilon_1)
$$
holds for all $k \in[1, \cdots, K]$.
\end{lemma} 
\begin{proof}
By Lemma \ref{simulation}, we have
\#\nonumber
&V^{\pi^{*},\widehat{\cM},r+\widehat{b}}-V^{\pi^{*},\cM,r} \\\nonumber
&=\sum_{h=0}^{H-1} \mathbb{E}_{(z_{h}, a_{h}) \sim d_{\widehat{\cM},h}^{\pi^{*}}}[\widehat{b}_{h}(z_{h}, a_{h})+\mathbb{E}_{o'\sim\mathbb{P}_{h}^{\widehat{\cM}}(\cdot\mid z_{h}, a_{h})}[V_{ h+1}^{\pi^{*},\cM,r}(z_{h+1}')]-\mathbb{E}_{o'\sim\mathbb{P}_{h}^{\cM}(\cdot\mid z_{h}, a_{h})}[V_{ h+1}^{\pi^{*},\cM,r}(z_{h+1}')]] 
\\\nonumber
&\geq \sum_{h=0}^{H-1} \mathbb{E}_{(z_{h}, a_{h}) \sim d_{\widehat{\cM}, h}^{\pi^{*}}}\biggl[\min (c \alpha_{k}\|\widehat{\phi}_{h}(z, a)\|_{\Sigma_{\rho_{h}, \widehat{\phi}_{h}}^{-1}}, 1)+\mathbb{E}_{o'\sim\mathbb{P}_{h}^{\widehat{\cM}}(\cdot\mid z_{h}, a_{h})}[V_{ h+1}^{\pi^{*},\cM,r}(z_{h+1}')]\\ &\qquad\quad-\mathbb{E}_{o'\sim\mathbb{P}_{h}^{\cM}(\cdot\mid z_{h}, a_{h})}[V_{ h+1}^{\pi^{*},\cM,r}(z_{h+1}')]\biggl]
\\\nonumber &\geq \sum_{h=0}^{H-1} \mathbb{E}_{(z_{h}, a_{h}) \sim d_{\widehat{\cP}, h}^{\pi^{*}}}\biggl[\min (c \alpha_{k}\|\widehat{\phi}_{h}(z, a)\|_{\Sigma_{\rho_{h}, \widehat{\phi}_{h}}^{-1}}, 1)+\mathbb{E}_{o'\sim\mathbb{P}_{h}^{\widehat{\cP}}(\cdot\mid z_{h}, a_{h})}[V_{ h+1}^{\pi^{*},\cM,r}(z_{h+1}')]\\ \label{eq:almost3}&\qquad\quad-\mathbb{E}_{o'\sim\mathbb{P}_{h}^{\cP}(\cdot\mid z_{h}, a_{h})}[V_{ h+1}^{\pi^{*},\cM,r}(z_{h+1}')]\biggl]-\cO(H^2\epsilon_1),
\#
where in the fist inequality, we replace empirical covariance by population covariance by Lemma \ref{bonus_concentration}, here $c$ is an absolute constant, the second inequality is by Lemma \ref{transfer}. We define
$$
g_{h}(z, a)=\mathbb{E}_{o_h'\sim \mathbb{P}_{h}^{\widehat{\cP}}(\cdot \mid z, a)}[V_{ h+1}^{\pi^{*},\cM,r}(c(z,a,o'_h)
)]-\mathbb{E}_{o_h'\sim\mathbb{P}_{h}^{\cP}(\cdot\mid z, a)}[V_{ h+1}^{\pi^{*},\cM,r}(c(z,a,o'_h))]
$$
Notice that we have $\|g_{h}\|_{\infty} \leq 1$.
With Lemma \ref{guarantee_POMDP}, for any $(z,a)$ we have 
\#\label{eq:g_pomdp}
\mathbb{E}_{(z, a) \sim \rho_{h}}[g_{h}^{2}(z,a)] \leq \zeta_{k},~
\mathbb{E}_{(z, a) \sim \beta_{h}}[g_{h}^{2}(z,a)] \leq \zeta_{k}
\#
By Lemma \ref{$L$-step learn no}, we have:
\$&\sum_{h=0}^{H-1} \mathbb{E}_{(z, a) \sim d_{\widehat{\cP}, h}^{\pi^{*}}}[g_{h}(z, a)]+\cO(H^2\epsilon_1)\\\leq &\sum_{h=1}^H \mathbb{E}^{\widehat{\cP}}_{z_{h-L-1},a_{h-L-1}\sim \pi}\|\widehat{\phi}^{\top}(z_{h-L-1}, a_{h-L-1})\|_{\rho_{h-L-1}^{-1},\widehat{\phi}_{h-L-1}}\\& \qquad\quad \cdot \sqrt{|A|^{L} k\cdot\mathbb{E}_{(\tilde{z}_{h},\tilde{a}_{h}) \sim \beta_{h}}\{[g(\tilde{z}_{h},\tilde{a}_{h})]^{2}\}+B^{2} \lambda_{k} d+k\epsilon_1}+\cO(H^2\epsilon_1),
\$
where the first inequality is by Lemma \ref{masa} and the second inequality is by Lemma \ref{$L$-step learn}. Hence we have 
\#\nonumber&\sum_{h=0}^{H-1} \mathbb{E}_{(z, a) \sim d_{\widehat{\cP}, h}^{\pi^{*}}}[g_{h}(z, a)]\\&
\leq \nonumber\sum_{h=0}^{H-1} \min\{ 1,\mathbb{E}^{\widehat{\cP}}_{z_{h-L-1},a_{h-L-1}\sim \pi}\|\widehat{\phi}^{\top}(z_{h-L-1}, a_{h-L-1})\|_{\rho_{h-L-1}^{-1},\widehat{\phi}_{h-L-1}}\\\nonumber& \qquad\quad \cdot \sqrt{|A|^{L} k\zeta_k+B^{2} \lambda_{k} d+k\epsilon_1}\}+\cO(H^2\epsilon_1)
\\&\label{eq:almost4}\leq   \sum_{h=0}^{H-1}\min\{1,c\alpha_k \mathbb{E}^{\widehat{\cM}}_{z_{h-L-1},a_{h-L-1}\sim \pi}\|\widehat{\phi}^{\top}(z_{h-L-1}, a_{h-L-1})\|_{\rho_{h-L-1}^{-1},\widehat{\phi}_{h-L-1}}\}+\cO(H^2\epsilon_1),\#
where the first inequality is by \eqref{eq:g_pomdp}, in the last step we use Lemma \ref{transfer} and the definition
$$
\alpha_{k}=\sqrt{k|\mathcal{A}|^{L} \zeta_{k}+4 \lambda_{k} d+k\epsilon_1} / c.
$$
For $h\leq 0$, we have 
\#\label{eq:h<0 for pomdp}
\|\widehat{\phi}_h^{\top}(z_h,a_h)\|_{\rho_h^{-1},\widehat{\phi}_h}=\sqrt{\frac{1}{k+\lambda}}<\frac{1}{\sqrt{k}},
\#
since $\phi(s,a)=e_1$ for $h\leq 0$.

Combine \eqref{eq:almost1}, \eqref{eq:almost4} and \eqref{eq:h<0 for pomdp}, we conclude the proof.
\end{proof}
\begin{theorem}\label{theorem:M-regret for general}
With probability $1-\delta$, we have
$$
\sum_{k=1}^K V^{\pi^{*},\cM,r}-V^{\pi^k,\cM,r}  \leq \cO\big(H^{2}|\mathcal{A}|^{L} d^{2} K^{1/ 2} \log (d K|\cF| / \delta)^{1 / 2}+H^{2}K d^{1/2}\log (d K|\cF| / \delta)^{1 / 2}\epsilon_1 )\big).
$$
\end{theorem}
\begin{proof}
The proof is similar to Theorem \ref{theorem:M-regret}, we condition on the event that the MLE guarantee \ref{guarantee_POMDP} holds, which happens with probability $1-\delta$.
For fixed  $k$ we have
\$
& V^{\pi^{*},\cM,r}-V^{\pi^{k},\cM,r}
\\&\leq V^{\pi^{*},\cM,r+\widehat{b}^k}-V^{\pi^{k},\cM,r}+\frac{c\alpha_kL}{\sqrt{k}} +\cO(H^2\epsilon_1)
\\&\leq V^{\pi^k,\widehat{\cM},r+\widehat{b}^k}-V^{\pi^{k},\cM,r}+\frac{c\alpha_kL}{\sqrt{k}} +\cO(H^2\epsilon_1)\\
& =\sum_{h=0}^{H-1}\bigg[ \mathbb{E}_{(z_{h}, a_{h}) \sim d_{h}^{\pi^{k},\cM}}\big[\widehat{b}_{h}(z_{h}, a_{h})+\mathbb{E}_{o_h'\sim\mathbb{P}_{h}^{\widehat{\cM}}(o_{h}' \mid z_{h}, a_{h})}[V_{h+1}^{\pi^{k},\widehat{\cM},r+\widehat{b}^k}(z_{h+1}')]\\ &\qquad\quad-\mathbb{E}_{o_h'\sim \mathbb{P}_{h}^{\cM}(o_{h}' \mid z_{h}, a_{h})}[V_{h+1}^{\pi^{k},\widehat{\cM},r+\widehat{b}^k}(z_{h+1}')]\big]\bigg]+\frac{c\alpha_kL}{\sqrt{k}}+\cO(H^2\epsilon_1) .
\$
The first inequality comes from Lemma \ref{lem:almost_pomdp}, the second inequality comes from $\pi^k=\operatorname{argmax}_{\pi} V^{\pi,\widehat{\cM},r+\widehat{b}^k}$, and
the last equation comes from Lemma \ref{simulation}.

By Lemma \ref{transfer}, we have
\$
&\sum_{h=0}^{H-1}\bigg[ \mathbb{E}_{(z_{h}, a_{h}) \sim d_{h}^{\pi^{k},\cM}}\big[\widehat{b}_{h}(z_{h}, a_{h})+\mathbb{E}_{o_h'\sim\mathbb{P}_{h}^{\widehat{\cM}}(o_{h}' \mid z_{h}, a_{h})}[V_{h+1}^{\pi^{k},\widehat{\cM},r+\widehat{b}^k}(z_{h+1}')]\\ &\qquad\quad-\mathbb{E}_{o_h'\sim \mathbb{P}_{h}^{\cM}(o_{h}' \mid z_{h}, a_{h})}[V_{h+1}^{\pi^{k},\widehat{\cM},r+\widehat{b}^k}(z_{h+1}')]\big]\bigg]+\frac{c\alpha_kL}{\sqrt{k}}\\&\leq \sum_{h=0}^{H-1}\bigg[ \mathbb{E}_{(z_{h}, a_{h}) \sim d_{h}^{\pi^{k},\cP}}\big[\widehat{b}_{h}(z_{h}, a_{h})+\mathbb{E}_{o_h'\sim\mathbb{P}_{h}^{\widehat{\cP}}(o_{h}' \mid z_{h}, a_{h})}[V_{h+1}^{\pi^{k},\widehat{\cM},r+\widehat{b}^k}(z_{h+1}')]\\ &\qquad\quad-\mathbb{E}_{o_h'\sim \mathbb{P}_{h}^{\cP}(o_{h}' \mid z_{h}, a_{h})}[V_{h+1}^{\pi^{k},\widehat{\cM},r+\widehat{b}^k}(z_{h+1}')]\big]\bigg]+\frac{c\alpha_kL}{\sqrt{k}}+\cO(H^2\epsilon_1) 
\$

Denote
\$
f_h(z_h,a_h)=\frac{1}{2H+1}\biggl(\mathbb{E}_{o_h'\sim\mathbb{P}_{h}^{\widehat{\cP}}(o_{h}' \mid z_{h}, a_{h})}[V_{h+1}^{\pi^{k},\widehat{\cM},r+\widehat{b}^k}(z_{h+1}')]-\mathbb{E}_{o_h'\sim \mathbb{P}_{h}^{\cP}(o_{h}' \mid z_{h}, a_{h})}[V_{h+1}^{\pi^{k},\widehat{\cM},r+\widehat{b}^k}(z_{h+1}')]\biggl).
\$

Note that $\|\widehat{b}\|_{\infty} \leq 1$, hence we have $\|V_{ h+1}^{\pi^k,\widehat{\cM},r+b}\|_{\infty} \leq(2 H+1)$. Combining this fact with the above expansion, we have
\#\nonumber
V^{\pi^{*},\cM,r}-V^{\pi^k,\widehat{\cM},r}&=\sum_{h=0}^{H-1} \mathbb{E}_{(z_{h}, a_{h}) \sim d_{h}^{\pi^k,\cP}}[\widehat{b}_{h}(z_{h}, a_{h})]+(2 H+1) \sum_{h=0}^{H-1} \mathbb{E}_{(z_{h}, a_{h}) \sim d_{ h}^{\pi^k,\cP}}[f_{h}(z_{h}, a_{h})]\\\nonumber&\quad\qquad+\frac{c\alpha_kL}{\sqrt{k}}+\cO(H^2\epsilon_1)\\&\nonumber\leq\sum_{h=0}^{H-1} \mathbb{E}_{(z_{h}, a_{h}) \sim d_{h}^{\pi^k,\cP}}[\widehat{b}_{h}(z_{h}, a_{h})]+(2 H+1) \sum_{h=0}^{H-1} \mathbb{E}_{(z_{h}, a_{h}) \sim d_{ h}^{\pi^k,\cP}}[f_{h}(z_{h}, a_{h})]\\\label{eq:8pomdp}&\quad\qquad+\frac{c\alpha_kL}{\sqrt{k}}+\cO(H^2\epsilon_1 ).
\#
First, we calculate the bonus term in \eqref{eq:8}. We have
\$&\sum_{h=0}^{H-1} \mathbb{E}_{(z_{h}, a_{h}) \sim d_{h}^{\pi^{k},\cP}}[\widehat{b}_{h}(z_{h}, a_{h})]\\&\leq \sum_{h=0}^{H-1} \mathbb{E}_{(\tilde{z}, \tilde{a}) \sim d_{h-L}^{\pi^k,\cP}}\|\phi_{h-L}^{*}(\tilde{z}, \tilde{a})\|_{\Sigma_{\gamma_{h-L}, \phi_{h-L}^{*}}^{-1}} \sqrt{k|\mathcal{A}|^L \mathbb{E}_{(z, a) \sim \rho_{h}}[(\widehat{b}_{h}(z, a))^{2}]+4 \lambda_{k}d+k\epsilon_1} +2H\epsilon_1
%+\sqrt{|\mathcal{A}|^L \mathbb{E}_{(z, a) \sim \rho_{0}}[(\widehat{b}_{0}(z, a))^{2}]} 
,\$
where the inequality is following Lemma \ref{$L$-step no} associate with $\|\widehat{b}_{h}\|_{\infty} \leq 1$.

Note that we use the fact that $B=2$ when applying Lemma \ref{$L$-step no}. In addition, we have that for any $h \in[H]$,
$$
k \mathbb{E}_{(z, a) \sim \rho_{h}}[\|\widehat{\phi}_{h}(z, a)\|_{\Sigma_{\rho_{h}, \widehat{\phi}_{h}}^{-1}}^{2}]=k \operatorname{Tr}(\mathbb{E}_{\rho_{h}}[\widehat{\phi}_{h} \widehat{\phi}_{h}^{\top}]\{k \mathbb{E}_{\rho_{h}}[\widehat{\phi}_{h} \widehat{\phi}_{h}^{\top}]+\lambda_{k} I\}^{-1}) \leq d .
$$
Then we have
$$
\sum_{h=1}^{H} \mathbb{E}_{(z, a) \sim d_{h}^{\pi^k,\cP}}[b_{h}(z, a)] \leq \sum_{h=1}^{H} \mathbb{E}_{(\tilde{z}, \tilde{a}) \sim d_{h-L}^{\pi^k,\cP}}\|\phi_{h-L}^{*}(\tilde{z}, \tilde{a})\|_{\Sigma_{\rho_{h-L}, \phi_{h-L}^{*}}^{-1}} \sqrt{|\mathcal{A}|^L \alpha_{k}^{2} d+4 \lambda_{k} d+k\epsilon_1}+2H\epsilon_1.
%+\sqrt{|\mathcal{A}|^L \alpha_{1}^{2} d / k}
$$
Now we bound the second term in  \eqref{eq:8pomdp}. Following Lemma \ref{$L$-step no}, with $\|f_{h}(z, a)\|_{\infty} \leq 1$, we have
\$&\sum_{h=0}^{H-1} \mathbb{E}_{(z_{h}, a_{h}) \sim d_{h}^{\pi^k,\cP}}[f_{h}(z_{h}, a_{h})]\\&\leq \sum_{h=0}^{H-1} \mathbb{E}_{(\tilde{z}, \tilde{a}) \sim d_{h-L}^{\pi^k,\cP}}\|\phi_{h-L}^{*}(\tilde{z}, \tilde{a})\|_{\Sigma_{\gamma_{h-L}, \phi_{h-L}^{*}}^{-1}} \sqrt{k|\mathcal{A}|^L \mathbb{E}_{(z, a) \sim \rho_{h}}[f_{h}^{2}(z, a)]+4 \lambda_{k} d+k\epsilon_1}
%+\sqrt{|\mathcal{A}| \mathbb{E}_{(z, a) \sim \rho_{0}}[f_{0}^{2}(z, a)]} 
\\ &\leq \sum_{h=0}^{H-1} \mathbb{E}_{(\tilde{z}, \tilde{a}) \sim d_{h-L}^{\pi^k,\cP}}\|\phi_{h-L}^{*}(\tilde{z}, \tilde{a})\|_{\Sigma_{\gamma_{h-L}, \phi_{h-L}^{*}}^{-1}} \sqrt{k|\mathcal{A}|^L (\zeta_{k})+4 \lambda_{k} d+k\epsilon_1},
%+\sqrt{|\mathcal{A}| \zeta_{k}}
\$
where in the second inequality, we use $\mathbb{E}_{z, a \sim \rho_{h}}[f_{h}^{2}(z, a)] \leq \zeta_{k}$. Then we have
$$
\begin{aligned}
& V^{\pi^{*},\cM,r}-V^{\pi^{k},\cM,r} \\
& \leq \sum_{h=0}^{H-1} \mathbb{E}_{(z_{h}, a_{h}) \sim d_{h}^{\pi^k,\cM}}[b_{h}(z_{h}, a_{h})]+(2 H+1) \sum_{h=0}^{H-1} \mathbb{E}_{(z_{h}, a_{h}) \sim d_{h}^{\pi^k,\cM}}[f_{h}(z_{h}, a_{h})]+\frac{c\alpha_kL}{\sqrt{k}}+\cO(H^2\epsilon_1)\\
\leq & \sum_{h=0}^{H-1} \mathbb{E}_{(\tilde{z}, \tilde{a}) \sim d_{h-L}^{\pi^{k},\cP}}\|\phi_{h}^{*}(\tilde{z}, \tilde{a})\|_{\Sigma_{\gamma_{h-L}, \phi_{h-L}^{*}}^{-1}} \sqrt{|\mathcal{A}|^L \alpha_{k}^{2} d+4 \lambda_{k} d+k\epsilon_1}
%+\sqrt{|\mathcal{A}|^L \alpha_{1}^{2} d / k}+
\\&+(2 H+1) \sum_{h=0}^{H-1} \mathbb{E}_{(\tilde{z}, \tilde{a}) \sim d_{h-L}^{\pi^{k},\cP}}\|\phi_{h-L}^{*}(\tilde{z}, \tilde{a})\|_{\Sigma_{\gamma_{h-L}^{-1}, \phi_{h-L}^{*}}} \sqrt{k|\mathcal{A}|^L \zeta_{k}+4 \lambda_{k} d+k\epsilon_1}
%+(2 H+1) \sqrt{|\mathcal{A}|^L \zeta_{k}}
+\frac{c\alpha_kL}{\sqrt{k}}+\cO(H^2\epsilon_1).
\end{aligned}
$$
Hereafter, we take the dominating term out. First, recall
$$
\alpha_{k}=O\bigg(\sqrt{k|\mathcal{A}|^L \zeta_{k}+\lambda_{k} d} \bigg).
$$
Second, recall that $\gamma_{h}^{k}(z, a)=\frac{1}{k} \sum_{i=0}^{k-1} d_{h}^{\pi^{i}}(z, a)$. Thus
\$
& \sum_{k=1}^{K}\mathbb{E}_{(\tilde{z}, \tilde{a}) \sim d_{ h}^{\pi^k,\cP}}\|\phi^{*}(\tilde{z}, \tilde{a})\|_{\Sigma_{\gamma_{h}^{n}, \phi_{h}^{*}}^{-1}}\\&\quad
\leq \sqrt{K \sum_{k=1}^{K} \mathbb{E}_{(\tilde{z}, \tilde{a}) \sim d_{ h}^{\pi^k,\cP}}[\phi_{h}^{*}(\tilde{z}, \tilde{a})^{\top} \Sigma_{\gamma_{h}^{k}, \phi_{h}^{*}}^{-1} \phi_{h}^{*}(\tilde{z}, \tilde{a})]}\\&\quad\leq  \sqrt{K\bigg(\log \operatorname{det}\bigg(\sum_{k=1}^{K} \mathbb{E}_{(\tilde{z}, \tilde{a}) \sim d_{h}^{\pi^k,\cP}}[\phi_{h}^{*}(\tilde{z}, \tilde{a}) \phi_{h}^{*}(\tilde{z}, \tilde{a})^{\top}]\bigg)-\log \operatorname{det}(\lambda_{1} I)\bigg)}\\& \quad\leq \sqrt{d K \log \bigg(1+\frac{K}{d \lambda_{1}}\bigg)},
\$
where the first inequality is by Cauchy-Schwarz inequality, the second inequality is by Lemma \ref{lemma14} and the third inequality is by Lemma \ref{potential}.

Finally, The MLE guarantee gives
$$
\zeta_{k}=O\bigg(\frac{\log (d k|\cF| / \delta)}{k}\bigg).
%+\epsilon_1^2)
$$
Combining all of the above, we have
\$
 \sum_{k=1}^{K}V^{\pi^{*},\cM,r}-V^{\pi^{k},\cM,r} &\leq \cO\big(H^{2}|\mathcal{A}|^{L} d^{2} K^{1/ 2} \log (d K|\cF| / \delta)^{1 / 2} +H^{2}K d^{1/2}\log (d K|\cF| / \delta)^{1 / 2}\epsilon_1 )\big),
\$
which concludes the proof.
\end{proof}
\begin{lemma}
With probability $1-\delta$, we have 
\$
V^{\pi^{*},\cP,r}-V^{\pi^k,\cP,r}\leq \cO\big(H^{2}|\mathcal{A}|^{L} d^{2} K^{1/ 2} \log (d K|\cF| / \delta)^{1 / 2} +H^{2}K d^{1/2}\log (d K|\cF| / \delta)^{1 / 2}\epsilon_1 )\big).
\$
\end{lemma}
\begin{proof}
Combined Lemma \ref{dif} and Theorem \ref{theorem:M-regret for general}, we conclude the proof.
\end{proof}
\section{Experiment Details}\label{sed:details}
In our experiment, we assess the performance of the proposed algorithm on the partially-observed diabolical combination lock (pocomblock) problem, characterized by a horizon $H$ and a set of $10$ actions. At each temporal stage $h$, there exist three latent states $s_{i ; h}$ for $i \in \{0,1,2\}$. We denote the states $s_{i ; h}$ for $i \in \{0,1\}$ as favorable states and $s_{2 ; h}$ as unfavorable states. For each $s_{i ; h}$ with $i \in \{0,1\}$, a specific action $a_{i ; h}^*$ is randomly selected from the 10 available actions.
When the agent is in state $s_{i ; h}$ for $i \in \{0,1\}$ and performs action $a_{i ; h}^*$, it transitions to states $s_{0 ; h+1}$ and $s_{1 ; h+1}$ with equal likelihood. Conversely, executing any other actions will deterministically lead the agent to state $s_{2 ; h+1}$. In the unfavorable state $s_{2 ; h}$, any action taken by the agent will inevitably result in transitioning to state $s_{2 ; h+1}$.

As for the reward function, a reward of $1$ is assigned to states $s_{i ; H}$ for $i \in \{0,1\}$, signifying that favorable states at horizon $H$ yield a reward of $1$. Additionally, with a 0.5 probability, the agent receives an anti-shaped reward of $0.1$ upon transitioning from a favorable state to an unfavorable state. All other states and transitions yield a reward of zero.
When the step $h$ is odd, the observation $o$ is generated with a dimension of $2^{\lceil\log (H+4)\rceil}$ by concatenating the one-hot vectors of latent state $s$ and horizon $h$, introducing Gaussian noise sampled from $\mathcal{N}(0,0.1)$ for each dimension, appending 0 if needed, and multiplying with a Hadamard matrix. For even steps $h$, the observations corresponding to one of the good states and the bad states are identical, the other good state's observation function is the same as the time step is odd. The initial state distribution is uniformly distributed across $s i ; 0$ for $i \in \{0,1\}$ . We employ a two-layer neural network to capture the essential features of the problem.

It is noteworthy that the optimal policy consists of selecting the specific action $a_{i ; h}^*$ at each step $h$. Once the agent enters an unfavorable state, it remains trapped in the unfavorable state for the entire episode, thus failing to attain the substantial reward signal at the conclusion. This presents an exceptionally demanding exploration problem, as a uniform random policy offers a mere $10^{-H}$ probability of achieving the objectives.
In our experiment, we compare our method with BRIEE, the latest representation learning algorithm for MDP. In particular, we make modifications to BRIEE to take a sequence of observations as input for representation learning, in order to work in the POMDP settings.
\begin{figure}[H]
  \centering
  
    \includegraphics[width=0.45\textwidth]{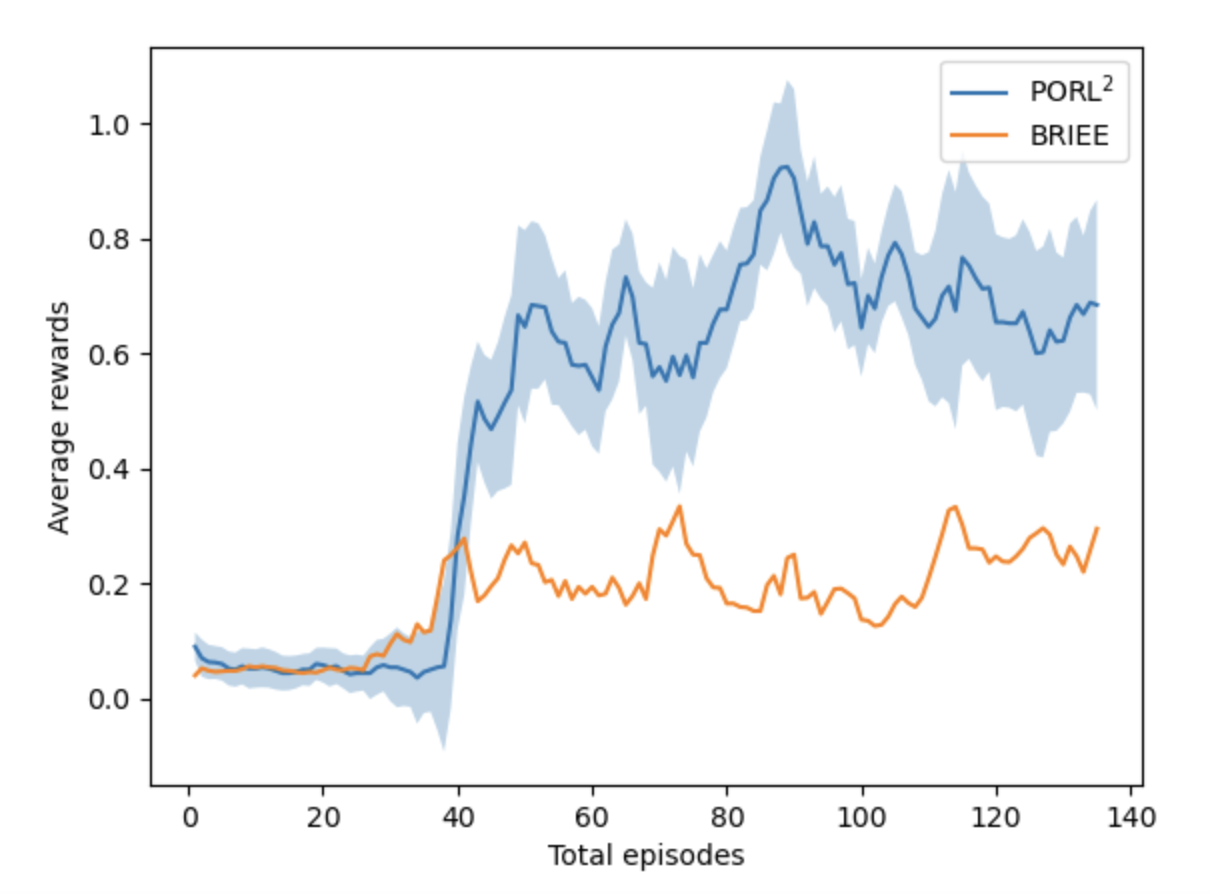}
    \caption{Moving average of evaluation returns of pocomblock for PORL$^2$ and BRIEE}
    \label{fig2:subfig1}
\end{figure}

\paragraph{Implementation details for PORL$^2$} In our implementation, we extend the BRIEE framework by considering two consecutive observations (pocomblock is $2$-decodable), $o_{h-1}$ and $o_h$, as a single variable, denoted as $z_h$. When employing LSVI-LLR, we compute the value function based on the state instead of $z$. This choice is justified by the fact that the rewards associated with different observations of the same state are equal. Figure \ref{fig2:subfig1} is the moving average of evaluation returns of pocomblock for PORL$^2$ and BRIEE.

We record the hyperparameters we try and the final hyperparameter we use for PORL$^2$ in Table \ref{tab:PORL} and BRIEE in Table \ref{tab:briee}.

\begin{table}[h]
\centering
\begin{tabular}{ccc} 
\toprule
& Value Considered & Value \\
\hline
 Batch size  & \{256, 512\} & 512\\
 Discriminator $ f$  number of gradient steps  & \{4,8,16,32\} & 8\\
Horizon & \{4,6,7,10\} & 7\\
Seeds & \{1,12,123,1234,12345\}& 12345\\
Decoder $\phi$ number of gradient steps  & \{4,8,16,32\} & 8 \\
The number of iterations of representation learning & \{6,8,10,12\}& 10 \\
LSVI-LLR bonus coefficient  $\beta$ & \{1\} & 1 \\
\text { LSVI-LLR regularization coefficient } $\lambda$ & \{1\} & 1  \\

 Optimizer  & \{ SGD \} & SGD \\
Decoder  $\phi$ learning rate  & \{1e-3,1e-2\} & 1e-2 \\
Discriminator $f$ learning rate  & \{1e-3,1e-2\} & 1e-2 \\
\toprule
\end{tabular}
\caption{Hyperparameters for PORL$^2$}
\label{tab:PORL}
\end{table}

\begin{table}[h]
\centering
\begin{tabular}{ccc} 
\toprule
& Value Considered & Value \\
\hline
Batch size  & \{256, 512\} & 512\\
 Discriminator $ f$  number of gradient steps  & \{4,8,16,32\} & 8\\
Horizon & \{4,6,7,10\} & 7\\
Seeds & \{1,12,123,1234,12345\}& 12345\\
Decoder $\phi$ number of gradient steps  & \{8\} & 8 \\
The number of iterations of representation learning & \{10\}& 10 \\
LSVI-LLR bonus coefficient  $\beta$ & \{1,10\} & 10 \\
 LSVI-LLR regularization coefficient  $\lambda$ & \{1\} & 1  \\

\text { Optimizer } & \{\text { SGD }\} & \text { SGD } \\
Decoder  $\phi$ learning rate  & \{1e-2\} & 1e-2 \\
Discriminator $f$ learning rate  & \{1e-2\} & 1e-2 \\
\toprule
\end{tabular}
\caption{Hyperparameters for BRIEE}
\label{tab:briee}
\end{table}
%%%%%%%%%%%%%%%%%%%%%%%%%%%%%%%%%%%%%%%%%%%%%%%%%%%%%%%%%%%%%%%%%%%%%%%%%%%%%%%
%%%%%%%%%%%%%%%%%%%%%%%%%%%%%%%%%%%%%%%%%%%%%%%%%%%%%%%%%%%%%%%%%%%%%%%%%%%%%%%

\end{document}